\documentclass[sigconf,nonacm]{acmart}

\usepackage{booktabs}       
\usepackage{amsfonts}       
\usepackage{nicefrac}       
\usepackage{microtype}      
\usepackage[export]{adjustbox}

\usepackage{mdframed}
\usepackage{amsmath,mathrsfs,mathtools,bm}
\usepackage{bbm}
\usepackage{scalerel}
\usepackage{amsthm}
\usepackage{fontawesome}
\usepackage{xcolor}
\usepackage{wrapfig}
\usepackage{setspace}
\usepackage{multirow}
\usepackage{bm,bbm}
\usepackage{paralist}
\usepackage{enumitem}
\usepackage{caption}[small]
\usepackage{subcaption} 
\usepackage{multicol}
\usepackage{float}
    \floatstyle{ruled}
    \newfloat{model}{thp}{lop}
    \floatname{model}{Model}
    
\usepackage[ruled,noresetcount,vlined,linesnumbered]{algorithm2e}

\newtheorem{problem*}{Problem}

\newtheorem{theorem}{Theorem}

\newtheorem{definition}{Definition}

\newtheorem*{example*}{Example}

\makeatletter

\newcommand{\nosemic}{\renewcommand{\@endalgocfline}{\relax}}
\newcommand{\dosemic}{\renewcommand{\@endalgocfline}{\algocf@endline}}
\let\oldnl\nl
\newcommand{\nonl}{\renewcommand{\nl}{\let\nl\oldnl}}

\SetKwFor{Repeat}{repeat}{:}{endw}
\makeatother

\DeclareMathOperator*{\argmax}{argmax}

\newcommand{\cE}{\mathcal{E}} 
\newcommand{\cG}{\mathcal{G}} 
 \newcommand{\cL}{\mathcal{L}}
\newcommand{\cM}{\mathcal{M}}

 \newcommand{\cY}{\mathcal{Y}}
 \newcommand{\cX}{\mathcal{X}}
\newcommand{\EE}{\mathbb{E}}

\newcommand\DCG{\textsl{DCG}}

\AtBeginDocument{%
  \providecommand\BibTeX{{%
    \normalfont B\kern-0.5em{\scshape i\kern-0.25em b}\kern-0.8em\TeX}}}

\setcopyright{acmcopyright}
\copyrightyear{2021}
\acmYear{}
\acmDOI{}
\acmSubmissionID{}

\begin{document}

\title{End-to-end Learning for Fair Ranking Systems}

\author{James Kotary}
\affiliation{%
  \institution{Syracuse University}
  \city{Syracuse, NY}
  \country{USA}
}
\email{jkotary@syr.edu}

\author{Ferdinando Fioretto}
\affiliation{%
  \institution{Syracuse University}
  \city{Syracuse, NY}
  \country{USA}
}
\email{ffiorett@syr.edu}

\author{Pascal Van Hentenryck}
\affiliation{%
  \institution{Georgia Institute of Technology}
  \city{Atlanta, GA}
  \country{USA}
}
\email{pvh@isye.gatech.edu}

\author{Ziwei Zhu}
\affiliation{%
  \institution{Texas A\&M University}
  \city{College Station, TX}
  \country{USA}
}
\email{zhuziwei@tamu.edu}

\renewcommand{\shortauthors}{Kotary and Fioretto, et al.}

\begin{abstract}
The learning-to-rank problem aims at ranking items to maximize exposure of those most relevant 
to a user query. A desirable property of such ranking systems is to guarantee some notion of fairness among specified item groups. 
While fairness has recently been considered in the context of learning-to-rank systems, current methods cannot provide guarantees on the fairness of the proposed ranking policies.
This paper addresses this gap and introduces \emph{Smart Predict and Optimize for Fair Ranking (SPOFR)}, an integrated optimization 
and learning framework for fairness-constrained learning to rank. The end-to-end SPOFR framework includes a constrained optimization sub-model and produces ranking policies that are guaranteed to satisfy fairness constraints, while allowing for fine control of the fairness-utility tradeoff. SPOFR is shown to significantly improve current state-of-the-art fair learning-to-rank systems with respect to established performance metrics.
\end{abstract}

\maketitle\sloppy\allowdisplaybreaks

\section{Introduction}

 Ranking systems are a pervasive aspect of our everyday lives: They 
 are an essential component of online web searches, job searches, property renting, streaming content, and even potential friendships. 
 In these systems, the items to be ranked are products, job candidates,
 or other entities associated with societal and economic benefits, and 
 the relevance of each item is measured by implicit feedback from users
 (click data, dwell time, etc.). 
 It has been widely recognized that the position of an item in the ranking has a strong influence on its exposure, selection, and,  
 ultimately economic success.

Surprisingly, though, the algorithms used to learn these rankings 
are typically oblivious to their potential 
disparate impact on the exposure of different groups of items. 
For example, it has been shown that in a job candidate ranking system, 
a small difference in relevance can incur a large difference in exposure 
for candidates from a minority group \cite{elbassuoni:19}. Similarly, 
in an image search engine, a disproportionate number of males may be 
shown in response to the query `CEO'~\cite{singh2018fairness}. 

Ranking systems that ignore fairness considerations, or are unable to bound these effects, are prone to the ``rich-get-richer'' dynamics that exacerbate the disparate impacts. The resulting biased rankings can be detrimental to users, ranked items, and ultimately society. 
{\em There is thus a pressing need to design learning-to-rank (LRT) systems
that can deliver accurate ranking outcomes while controlling disparate impacts.} 
  
Current approaches to fairness in learning-to-rank systems 
rely on using a loss function representing a weighted 
combination of expected task performance and fairness. 
This strategy is effective in improving the fairness of predicted 
rankings on average, but has three key shortcomings: 
\textbf{(1)} The resulting rankings, even when fair in expectation across 
all queries, can admit large fairness 
disparities for some queries. This aspect may contribute to exacerbate the rich-get-richer dynamics, while giving a false sense of controlling the  system's disparate impacts.
\textbf{(2)} While a tradeoff between fairness and ranking utility is usually 
desired, these models cannot be directly controlled through the specification 
of an allowed magnitude for the violation of fairness. 
\textbf{(3)} A large hyperparameter search is required to find 
the weights of the loss function that deliver the desired performance tradeoff.  Furthermore, each of these issues becomes worse as the number of protected groups increases.

This paper addresses these issues and proposes the first fair learning 
to rank system--named Smart Predict and Optimize for Fair Ranking (SPOFR)--that 
\emph{guarantees} satisfaction of fairness in the resulting rankings. 
The proposed framework uses a unique integration of a constrained 
optimization model within a deep learning pipeline, which is trained 
end-to-end to produce \emph{optimal} ranking policies with respect to 
empirical relevance scores, while enforcing fairness constraints. 
In addition to providing a certificate on the fairness requirements, 
SPOFR allows the modeler to enforce a large number of fairness concepts 
that can be formulated as linear constraints over ranking policies, 
including merit-weighted and multi-group fairness.

\textbf{Contributions}
The paper makes the following contributions:
\textbf{(1)}  
It proposes \emph{SPOFR}, a Fair LTR system that predicts and optimizes through an end-to-end composition of differentiable functions, guaranteeing the satisfaction of user-specified group fairness constraints.
\textbf{(2)} Due to their discrete structure, imposing fairness constraints 
over ranking challenges the computation and back-propagation of 
gradients. To overcome this challenge,  \emph{SPOFR} introduces a novel training scheme 
which allows direct optimization of empirical utility metrics on predicted 
rankings 
using efficient back-propagation through constrained optimization programs. 
\textbf{(3)} The resulting system ensures uniform fairness guarantees \emph{over all 
queries}, a directly controllable fairness-utility tradeoff, and guarantees 
for multi-group fairness.
\textbf{(4)} These unique aspects are demonstrated on several LTR datasets in the partial information setting. Additionally, SPOFR is shown to significantly improve current state-of-the-art fair LTR systems with respect to established performance metrics.

\section{Related Work}

Achieving fairness in ranking systems has been studied in multiple
contexts. The imposition of fairness
constraints over rankings is nontrivial and hinders common
gradient-based methods due to their discrete nature. To address this challenge, multiple notions of fairness in ranking have been
developed. \citet{celis2017ranking} propose to directly require fair representation between groups within each prefix of a ranking, by
specifying a mixed integer programming problem to solve for rankings of the desired form. ~\citet{zehlike2017fa} design a greedy randomized algorithm to produce rankings which satisfy fairness up to a threshold
of statistical significance. The approach taken by ~\citet
{singh2018fairness} also constructs a randomized ranking policy
by  formalizing the ranking policy as a solution to a linear optimization problem with constraints ensuring fair exposure between groups in expectation. 

Fairness in learning-to-rank research is much sparser. This problem is studied by \citet{zehlike2020reducing}, which adopts the LTR approach of \citet{cao2007learning} and introduces a penalty term to the loss function to account for the violation of group fairness in the top ranking position. Stronger fairness results are reported by \citet{yadav2021policy} and \citet{singh2019policy}, which apply a policy gradient method to learn fair ranking policies. The notion of fairness 
is enforced by a penalty to its violation in the  loss function, forming a weighted combination of terms representing fairness violation and ranking utility over rankings sampled from the learned polices 
using a REINFORCE algorithm~\cite{williams1992simple}.

While these methods constitute a notable contribution to the state of the art, by being penalty-based, they inevitably present some of the limitations discussed in the previous section. In contrast, this paper introduces a novel LTR technique for learning fair rankings by the use of constrained optimization programs trained end-to-end with deep learning \cite{kotary2021end}. By integrating constrained optimization in the training cycle of a deep learning model, the proposed solution \emph {guarantees} the satisfaction of fairness constraints, within a prescribed tolerance, in the resulting rankings. 

\begin{table}[!t]
    \centering
    \resizebox{0.95\linewidth}{!}{
     \begin{tabular}{l l} 
     \toprule
     \textbf{Symbol} & \textbf{Semantic}\\
     \midrule
     $N$      & Size of the training dataset\\
     $n$      & Number of items to rank, for each query\\
     $\bm{x}_q = (x^i_q)_{i=1}^n$ & List of items to rank for query $q$\\
     $\bm{a}_q = (a^i_q)_{i=1}^n$ & protected groups associated with items $x^i_q$\\
     $\bm{y}_q = (y^i_q)_{i=1}^n$ & relevance scores (1-hot labels)\\
     $\sigma$ & Permutation of the list $[n]$ (individual rankings)\\
     $\Pi$    & Ranking policy \\
     $\bm{v} = (v_i)_{i=1}^n$ & Position bias vector\\
     $\bm{w} = (w_{i})_{i=1}^n$ & Position discount vector \\
     \bottomrule
    \end{tabular}  
}
\caption{Common symbols}
\label{table:symbols}
\end{table}

\section{Settings and Goals}
\label{sec:settings}

The LTR task consists in learning a mapping
between a list of $n$ \emph{items} and a permutation $\sigma$ of the list $[n]$, which defines the order in which the items should be ranked in response to a user query. The LTR setting considers a training dataset $\bm{D} = (\bm{x}_q, \bm
{a}_q, \bm{y}_q)_{q=1}^{N}$ where the $\bm{x}_q \in \cX$ describe
lists $(x_q^i)_{i=1}^n$ of $n$ items to rank, with each item $x_q^i$
defined by a feature vector of size $k$. The $\bm{a}_q = (a_q^i)_
{i=1}^n$ elements describe protected group attributes in some domain $\cG$ for
each item $x_q^i$. The $\bm{y}_q \in \cY$ are
supervision labels $(y_q^i)_{i=1}^n$ that associate a non-negative value, called \emph{relevance scores}, with each item.
Each sample $\bm{x}_q$ and its corresponding label $\bm{y}_q$ in $\bm
{D}$ corresponds to a unique \emph{query} denoted $q$. For example, on a image web-search context, a query $q$ denotes the search keywords, e.g., ``nurse'', the feature vectors $x_q^i$ in $\bm{x}_q$ encode representations of the items relative to $q$, the associated protected group attribute $a_q^i$ may denote gender or race, and the label $y_q^i$ describes the relevance of item $i$ to query $q$.

The goal of learning to rank is to predict, for any query $q$, a distribution of rankings $\Pi$, called a \emph{ranking policy}, from which individual rankings can be sampled. 
The utility $U$ of a ranking policy $\Pi$ for query $q$ is defined as
\begin{equation}
    U(\Pi,q)   = \mathbb{E}_{\sigma \sim \Pi}\;
    \left[\Delta(\sigma, \bm{y}_q)\right],
    \label{eq:policy_utility}
\end{equation} 
where $\Delta$ measures the utility of a given ranking
with respect to given relevance scores $\bm{y}_q$. Note that, while, in isolation,  single lists of relevance scores on items may not be sufficient to rank them, in aggregate, these scores will provide enough information to predict a useful ranking notion.

Let $\mathcal
{M_\theta}$ be a machine learning model, with parameters $\theta$, which takes as input a query and returns a ranking policy.
The LTR goal is to find parameters $\theta^*$ that maximize the
empirical risk:
\begin{equation}
    \theta^*   = \argmax_{\theta} \;\;
    \frac{1}{N} \sum_{q = 1}^N U(\mathcal{M}_\theta(\bm{x}_q), \bm{y}_q).
    \tag{P}
    \label{eq:policy_ERM}
\end{equation}
This description refers to the \textit{Full-Information} 
setting~\cite{PropSVMRank:JoachimsSS17}, in which all target relevance scores are assumed to be known. While this setting is described to ease notation, the methods proposed in this work are not limited to this setting and Section \ref{sec:experiments} and Appendix \ref{sec:partial_info} discuss the \emph{Partial-Information} setting.

\subsubsection*{\bf Fairness} This paper aims at learning ranking policies that satisfy group fairness. It considers a predictor $\cM$
 satisfying some group fairness notion on the learned ranking
 policies with respect to protected attributes $\bm{a}_q$. A desired
 property of fair LTR models is to ensure that, for a given query,
 items associated with different groups receive equal exposure over
 the ranked list of items. The exposure $\mathcal{E}
 (i, \sigma)$ of item $i$ within some ranking $\sigma$ is
 a function of only its position, with higher positions receiving
 more exposure than lower ones. Thus, similar to \cite{singh2019policy}, this exposure is quantified by  $\mathcal{E}(i, \sigma) = v_{\sigma_i}$, where  the \emph{position bias} vector $\bm{v}$ is defined with elements $v_j = \nicefrac{1}{(1+j)^p}$, for $j \in [n]$ and with $p > 0$ being an  arbitrary power.

For ranking policy $\cM_\theta(\bm{x}_q)$ and
query $q$, fairness of exposure requires that, for every group indicator $g \in \cG$, 
$\cM$'s rankings are statistically independent of the protected 
attribute $g$:
\begin{equation}
\label{eq:fairness_of_exposure}
  \EE_{\substack{\sigma \sim \cM_\theta(\bm{x}_q) \\ i \sim [n]}} 
  \left[ \cE \left(i, \sigma \right) | a_q^i = g \right]
  = 
  \EE_{\substack{\sigma \sim \cM_\theta(\bm{x}_q)\\ i \sim [n]}} 
  \left[ \cE \left(i, \sigma \right) \right].
\end{equation}
This paper considers bounds on fairness defined as the difference between the group and population level terms, i.e.,
\begin{equation}
\label{eq:fairness_violation}
  \nu(\cM_\theta(\bm{x}_q), g) = 
  \EE \left[ \cE(i, \sigma) | a_q^i = g \right]
  - 
  \EE \left[ \cE (i, \sigma ) \right].
\end{equation}

\begin{definition}[$\delta$-fairness]
\label{def:delta_fairness}
A model $\cM_\theta$ is $\delta$-fair, with respect to exposure, 
if for any query $q \in [N]$ and group $g \in \cG$: 
\[
\left| \nu(\cM_\theta(\bm{x}_q), g) \right| \leq \delta. 
\]
In other words, the fairness violation on the resulting ranking 
policy is upper bounded by $\delta \geq 0$.
\end{definition}

The goal of this paper is to design accurate LTR models that guarantee 
$\delta$-fairness for any prescribed fairness level $\delta \geq 0$. 
As noted by \citet{agarwal:18} and \citet{Fioretto:ECAI20}, 
several fairness notions, including those considered 
in this paper, can be viewed as linear constraints between the properties 
of each group with respect to the population. While the above description 
focuses on exposure, the methods discussed in this paper can handle \emph{any} 
fairness notion that can be formalized as a (set of) linear constraints, 
including \emph{merit weighted fairness}, introduced in Section \ref{sec:optimize}. A summary of the common adopted symbols and their associated semantics is provided in Table \ref{table:symbols}. 

\section{Learning Fair Rankings: Challenges}
\label{section:learning_fair_ranking}

When interpreted as constraints of the form \eqref{eq:fairness_violation}, fairness properties can be explicitly imposed to problem \eqref{eq:policy_ERM}, resulting in a constrained empirical risk problem, formalized as follows:
\begin{subequations}
    \label{eq:fair_ERM}
\begin{align}
  \label{eq:fair_ERMa}
    \theta^*   = \argmax_{\theta} &\;\;
    \frac{1}{N} \sum_{q = 1}^N U(\mathcal{M}_\theta(\bm{x}_q), \bm{y}_q)\\
    \label{eq:fair_ERMb}
    \texttt{s.t.} &\;\;
    \left|\nu(\cM_\theta(\bm{x}_q), g)\right| \leq \delta \;\; \forall q \in [N], g \in \cG.
\end{align}
\end{subequations}
Solving this new problem, however, becomes challenging due to the 
presence of constraints. 
Rather than enforcing constraints (\ref{eq:fair_ERMb}) exactly, 
state of the art approaches 
in fair LTR (e.g., \cite{singh2018fairness,yadav2021policy})  rely on augmenting the loss function \eqref{eq:fair_ERMa} 
with a term that penalizes the constraint violations $\nu$ weighted by 
a multiplier $\lambda$.
This approach, however, has several undesirable properties:
\begin{enumerate}[leftmargin=*, parsep=0pt, itemsep=2pt, topsep=2pt]
    \item Because the constraint violation term is applied at the 
    level of the loss function, it applies only on average over the 
    samples encountered during training. {Because the sign ($\pm$) 
    of a fairness violation depends on which group is favored,  
    disparities in favor of one group can cancel out those 
    in favor of another group for different queries.} This can lead 
    to models which predict individual policies that are far 
    from satisfying fairness in expectation, as desired. These effects 
    will be shown in Section~\ref{sec:experiments}.
    \item The multiplier $\lambda$ must be treated as a hyperparameter, 
    increasing the computational effort required to find desirable solutions. 
    This is already challenging for binary groups and the task becomes 
    (exponentially) more demanding with the increasing of the number of 
    protected groups.

    \item When a tradeoff between fairness and utility is desired, it
     cannot be controlled by specifying an allowable magnitude for
     fairness violation. This is due to the lack of a reliable
     relationship between the hyperparameter $\lambda$ and the
     resulting constraint violations. In particular, choosing
     $\lambda$  to satisfy Definition \ref{def:delta_fairness} for a
     given $\delta$ is near-impossible due to the sensitivity and
     unreliability of the relationship between these two values.
\end{enumerate}

\noindent
The approach proposed in this paper avoids these difficulties by 
providing an end-to-end integration of predictions and optimization into a single machine-learning pipeline, where (1) fair policies are obtained by an optimization model using the predicted relevance scores and (2) the utility metrics are back-propagated from the loss function to the inputs, through the optimization model and the predictive models. This also ensures that the fairness constraints are satisfied on \emph{each} predicted ranking policy.

\section{SPOFR}
\label{section:spofr}

\begin{figure*}[!t]
\includegraphics[width=0.99\textwidth]{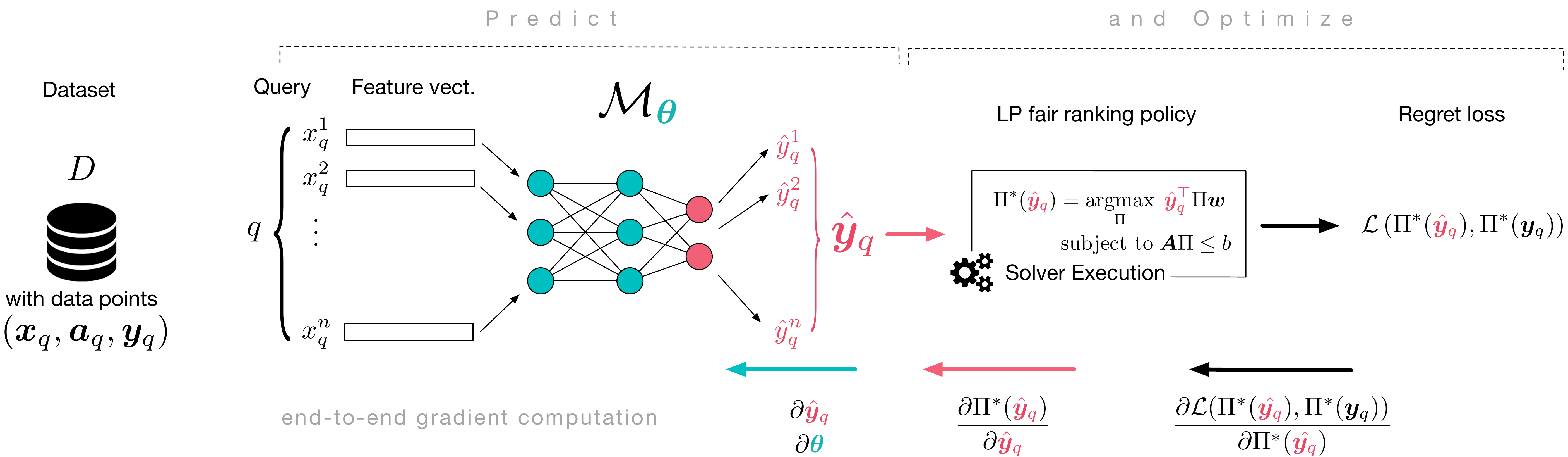}
\caption{SPOFR. A single neural networks learns to predict item scores from individual feature vectors, which are used to construct a linear objective function for the constrained program that produces a ranking policy.
\label{fig:schema}}
\end{figure*}

This section presents the proposed \emph{Smart Predict and Optimize
for Fair Ranking} (SPOFR) framework and analyzes the properties of the
resulting predicted ranking policies.

\smallskip\noindent\textbf{Overview.}
The underlying idea behind SPOFR relies on the realization
that constructing an optimal ranking policy $\Pi_q$ associated with
a query $q$ can be cast as a linear program (as detailed in the
next section) which relies only on the relevance scores $\bm
{y}_q$. The cost vector of the objective function of this program is
however not observed, but can be predicted from the feature vectors
$x_q^i$ ($i\in [n])$ associated with the item list $\bm
{x}_q$ to rank. The resulting framework thus operates into three
steps:
\begin{enumerate}[leftmargin=*, parsep=0pt, itemsep=2pt, topsep=-2pt]

\item First, for a given query $q$ and its associated item list $\bm
 {x}_q$, a neural network model $\cM_\theta$ is used to predict
 relevance scores $\hat{\bm{y}}_q = (\hat{y}_q^1, \ldots, \hat
 {y}_q^n)$; 

\item Next, the predicted relevance scores are used to specify the
 objective function of a linear program whose solution will result in
 a \emph{fair} optimal (with the respect to the predicted scores)
 ranking policy $\Pi^*(\hat{\bm{y}}_q)$; 

\item Finally, a regret function, which measures the loss of
 optimality relative to the \emph{true} optimal policy $\Pi^*(\bm
 {y}_q)$ is computed, and gradients are back-propagated along each
 step, including in the $\argmax$ operator adopted by the linear
 program, creating an end-to-end framework.
\end{enumerate}

\noindent
The overall scheme is illustrated in Figure \ref{fig:schema}. It is
important to note that, rather than minimizing a standard error
(such as a mean square loss) between the predicted quantities $\hat
{\bm{y}}_q$ and the target scores $\bm{y}_q$, SPOFR \emph{minimizes
directly a loss in optimality of the predicted ranking with respect
to the optimal ones}.
Minimizing this loss is however challenging as ranking are discrete
structures, which requires to back-propagate gradients through a linear program. These steps are examined in detail in the rest of this section.

While the proposed framework is general and can be applied to any \emph
{linear} utility metric $U$ for rankings (see Problem \eqref
{eq:policy_utility}), this section grounds the presentation on a widely adopted utility metric, the \emph{Discounted Cumulative Gain (DCG)}:
\begin{equation}
    \DCG(\sigma, \bm{y}_q) = \sum_{i=1}^n y_q^i w_{\sigma_i} ,
\end{equation}
where $\sigma$ is a permutation over $[n]$, $\bm{y}_q$ are the true 
relevance scores, and $\bm{w}$ is an arbitrary weighting vector over ranking positions, capturing the concept of 
\emph{position discount}. Commonly, and throughout this paper, 
$\nicefrac{w_{i} = {1}}{\log_2\left(1+i \right)}$. 
Note that following \cite{yadav2021policy,singh2019policy}, 
these discount factors are considered distinct from the position bias 
factors $\bm{v}$ used in the calculation of group exposure.

\subsection{Predict: Relevance Scores}
\label{sec:predict}

Given a query $q$ with a list of items $\bm{x}_q = (x_q^1,\ldots,x_q^n)$ 
to be ranked, the \textsl{predict} step uses a \emph{single} fully 
connected ReLU neural network $\cM_\theta$ acting on each individual 
item $x_q^i$ to predict a score $\hat{y}_q^i$ $(i=1,\ldots,n)$.
Combined, the predicted scores for query $q$ are denoted with $\hat{\bm{y}}_q$ and serve as the cost vector associated with the optimization problem solved in the next phase.

\subsection{Optimize: Optimal Fair Ranking Policies} 
\label{sec:optimize}

The predicted relevance scores $\hat{\bm{y}}_q$, 
combined with the constant position discount values $\bm{w}$, can be used to form a linear function that estimates the utility metric (\DCG{}) of a ranking policy. 
Expressing the utility metric as a linear function makes it possible to represent the optimization step of our end-to-end pipeline as a linear program.  

\begin{model}[!t]
{
    \caption{LP Computing the Fair Ranking Policy}
    \label{model:fair_rank}
    \vspace{-6pt}
    \begin{subequations}
    \begin{align}
    \label{eq:6a}
    \Pi^*(\hat{\bm{y}}_q) = {\argmax}_{\bm{\Pi}} &\;\; 
     \hat{\bm{y}}_q^\top \, \Pi \, \bm{w} \\
    \mbox{subject to:} &\;\;  
    \label{eq:6b}
             \sum_{j} {\Pi}_{ij} = 1 \;\; \forall i \!\in\! [n]\\
    \label{eq:6c}
            & \;\; \sum_{i} {\Pi}_{ij} = 1 \;\; \forall j  \!\in\! [n]\\
    \label{eq:6d}
            & \;\; 0 \leq  {\Pi}_{ij}  \leq 1 \;\; \forall i,j \!\in\! [n]\\
    \label{eq:6e}
            & \;\; 
            |\nu(\Pi, g)| \leq \delta \;\; \forall g \in \cG
    \end{align}
    \end{subequations}
    }
    \vspace{-12pt}
\end{model}

\subsubsection*{\bf Linearity of the Utility Function}
\label{sec:utility}
The following description omits subscripts ``${}_q$'' for readability. 
The references below to ranking policy $\Pi$ and relevance scores $\bm{y}$ are to be interpreted in relation to an underlying query~$q$.

Using the Birkhoff--von Neumann decomposition~\cite{birkhoff1940lattice}, any $n \times n$ doubly stochastic matrix ${\Pi}$\footnote{A slight abuse of notation is used to 
refer to $\Pi$ as a matrix of marginal probabilities encoding the homonyms 
ranking policy.} can be decomposed into a convex combination of at most $(n-1)^2+1$ permutation matrices $P^{(i)}$, each associated 
with a coefficient $\mu_i \leq 0$, which can then represent rankings 
$\sigma^{(i)}$ {under the interpretation $\bm{w}_{\sigma^{(i)}} = P^{(i)}\bm{w}$}. 
A ranking policy is inferred from the set of resulting convex 
coefficients $\mu_i$, which sum to one, forming 
a discrete probability distribution: each permutation has likelihood 
equal to its respective coefficient 
\begin{equation}
\label{eq:birkhoff}
{\Pi} = \sum_{i=1}^{(n-1)^2+1} \!\!\!\! \mu_i P^{(i)}.
\end{equation}
Next, note that any linear function on rankings can be formulated 
as a linear function on their permutation matrices, which can then be 
applied to any square matrix. 
In particular, applying the DCG operator to a doubly stochastic matrix ${\Pi}$
results in the expected DCG over rankings sampled from its inferred policy. 
Given item relevance scores $\bm{y}$:
\begin{align*}
\label{eq:birkhoff_dcg}
    \mathbb{E}_{\sigma \sim \Pi} \DCG(\sigma, \bm{y})  &   = \sum_{i=1}^{(n-1)^2+1} \!\!\!\!\mu_i \, \bm{y}^\top P^{(i)} \, \bm{w}  \\
    & = \bm{y}^\top \left(\sum_{i=1}^{(n-1)^2+1} \!\!\!\!\mu_i \,  P^{(i)} \right)\,\bm{w} 
      = \bm{y}^\top \Pi \, \bm{w} \tag{by Eq.~\eqref{eq:birkhoff}}.
\end{align*}

The expected DCG of a ranking sampled from a ranking policy $\Pi$ can thus be represented as a linear function on $\Pi$, which serves as the objective function for Model \ref{model:fair_rank}
(see Equation \eqref{eq:6a}).
\emph{This analytical evaluation of expected utility is key to optimizing 
fairness-constrained ranking policies in an end-to-end manner}. 

Importantly, and in contrast to state-of-the art methods, this approach does not require sampling from ranking policies during training in order to evaluate ranking utilities. Sampling is only required during deployment of the ranking model.

\subsubsection*{\bf Ranking Policy Constraints} 
Note that, with respect to \emph{any} such linear objective function,
the optimal fair ranking policy $\Pi^*$ can be found by solving
a linear program (LP). The linear programming model for optimizing fair
ranking \DCG{} functions is presented in Model \ref
{model:fair_rank}, which follows the formulations presented in \cite{singh2018fairness}.

The ranking policy predicted by the SPOFR model takes the form of
a doubly stochastic $n\times n$ matrix $\Pi$, in which ${\Pi}_{ij}$ represents the marginal 
probability that item $i$ takes
position $j$ within the ranking. The doubly stochastic form is
enforced by equality constraints which require each row and column of
$\Pi$ to sum to $1$. With respect to row $i$, these constraints
express that the likelihood of item $i$ taking any of $n$ possible
positions must be equal to $1$ (Constraints \eqref{eq:6b}). Likewise, 
the constraint on column $j$ says that the total probability of some 
item occupying position $j$ must also be $1$ (Constraints \eqref{eq:6c}). 
Constraints \eqref{eq:6d}
require that each ${\Pi}_{ij}$ value be a probability. 

For the policy implied by $\Pi$ to be fair, additional \emph{fairness constraints} must be introduced. 

\smallskip\noindent{\bf Fairness Constraints.} 
\label{sec:fairness}
Enforcing fairness requires only one 
additional set of constraints, which ensures that the exposures
are allocated fairly among the distinct groups. The expected exposure of item $i$ in rankings $\sigma$ derived from 
a policy matrix $\Pi$ can be expressed in terms of position bias factors 
$\bm{v}$ as  
\(
  \EE_{\sigma \sim \Pi} \cE(i, \sigma) = \sum_{j=1}^n \Pi_{ij} v_j.
\)
The $\delta$-fairness of exposure constraints associated with predicted 
ranking policy $\Pi$ and group $g \in \cG$ becomes:
\begin{equation}
\label{eq:fairness_one_group}
    \left| \left(\frac{1}{|G_q^g|}  \mathbbm{1}_{g}  - \frac{1}{n}  
    \mathbbm{1} \right)^\top \Pi \, \bm{v} \right| \leq \delta,
\end{equation}
where, $G_q^g = \{ i : (a_q^i)=g   \}$, $\mathbbm{1}$ is the vector of 
all ones, and $\mathbbm{1}_g$ is a vector whose values equal to $1$ 
if the corresponding item to be ranked is in $G_q^g$, and $0$ otherwise. 
This definition is consistent with that of Equation \eqref{eq:fairness_of_exposure}. 
It is also natural to consider a notion of weighted fairness of exposure: 
\begin{equation}
\label{eq:fairness_one_group_weighted}
    \left| \left(\frac{\mu}{|G_q^g|}  
    \mathbbm{1}_{g}  - \frac{\mu_g}{n}  \mathbbm{1} \right)^\top \Pi\, \bm{v} \right| \leq \delta,
\end{equation} which specifies that group $g$ receive exposure in
 proportion to the weight $\mu_g$. In this paper, where applicable
 and for a notion of \emph{merit-weighted fairness of
 exposure}, $\mu_g$ is chosen to be the average relevance score of items in group $g$, while $\mu$ is the average over all items.  Note that, while the above  are natural choices for
 fairness in ranking systems, \emph{any linear constraint can be
 used instead}.

\subsection{Regret Loss and SPO Training}
The training of the end-to-end fair ranking model uses a loss function 
that minimizes the \emph{regret} between the exact and approximate policies, i.e., 
\begin{equation}
\label{eq:dcg_regret}
\cL(\bm{y}, \hat{\bm{y}}) = \bm{y}^\top \Pi^*(\bm{y})\, \bm{w} - 
                            \bm{y}^\top \Pi^*(\hat{\bm{y}})\, \bm{w}.
\end{equation}
To train the model with stochastic 
gradient descent, the main challenge is the back-propagation
through Model (\ref{model:fair_rank}), i.e., the final operation 
of our learnable ranking function. It is well-known that a parametric linear program with 
fixed constraints is a nonsmooth mapping from objective coefficients to 
optimal solutions. \emph{Nonetheless, effective approximations for the 
gradients of this mapping can be found~\cite{elmachtoub2020smart}}. 
Consider the optimal solution to a linear programming problem with 
fixed constraints, as a function of its cost vector $\hat{\bm{y}}$:
\begin{align*}
      \Pi^*(\hat{\bm{y}}) = \argmax_{\Pi} 
      &\;\; \hat{\bm{y}}^\top \Pi \\
      \texttt{s.t.} &\;\;  \bm{A} \Pi \leq b,
  \end{align*}
with $A$ and $b$ being an arbitrary matrix and vector, respectively. Given candidate costs $\hat{\bm{y}}$, the resulting optimal solution 
$\Pi^*(\hat{\bm{y}})$ can be evaluated relative to a known cost vector 
$\bm{y}$. Further, the resulting objective value can be compared 
to that of the optimal objective under the known cost vector using the regret
metric $\cL(\bm{y}, \hat{\bm{y}})$. 

The regret measures the loss in objective value, relative to the true cost
function, induced by the predicted cost. It is used as a loss function
by which the predicted linear program costs vectors can be supervised by
ground-truth values. 
However, the regret function is discontinuous with respect to $\hat{\bm{y}}$ for fixed $\bm{y}$. Following the approach 
pioneered in \cite{elmachtoub2020smart}, this paper uses a convex 
surrogate loss function, called the $\texttt{SPO+}$ loss, 
which forms a convex upper-bounding function over $\cL(\bm{y}, \hat{\bm{y}})$. 
Its gradient is computed as follows:
\begin{equation}
    \label{eq:SPO_regret_grad}
    \frac{\partial }{\partial \bm{y}} \cL(\bm{y}, \hat{\bm{y}}) \approx 
    \frac{\partial }{\partial \bm{y}} \mathcal{L}_{\texttt{SPO+}}(\bm{y}, \hat{\bm{y}}) 
    = \Pi^*(2\hat{\bm{y}} - \bm{y}) - \Pi^*(\bm{y}).
\end{equation}
Remarkably, risk bounds about the SPO+ loss relative to the SPO loss 
can be derived \cite{liu2021risk}, and the empirical minimizer of the SPO+  loss is shown to achieve low excess true risk with high probability. 
Note that, by definition, $\bm{y}^\top \Pi^*(\bm{y}) \geq  \bm{y}^\top \Pi^*(\hat{\bm{y}})$ 
and therefore $ \cL(\bm{y}, \hat{\bm{y}}) \geq 0 $. Hence, finding the $\hat{\bm{y}}$ minimizing $ \cL(\bm{y}, \hat{\bm{y}}) $ is equivalent to
finding the $\hat{\bm{y}}$ maximizing $ \bm{y}^\top \Pi^*(\hat{\bm{y}})$, since  $\bm{y}^\top \Pi^*(\bm{y})$ is a constant value.

\begin{algorithm}[!tb]
  \caption{Training the Fair Ranking Function}
  \label{alg:learning}
  \SetKwInOut{Input}{input}
  \Input{${D}, \alpha, \bm{w}:$ Training Data, Learning Rate, Position Discount.\!\!\!\!\!\!\!\!\!\!}
  \For{epoch $k = 0, 1, \ldots$} 
  { 
    \ForEach{$(\bm{x}, \bm{a}, \bm{y}) \!\gets\! {D}$ }
    {
        $\hat{\bm{y}} \gets \mathcal M_{\bm{\theta}}(\bm{x})$\\
        $\Pi_1 \gets \Pi^*(\bm{y}^\top \bm{w})$ by Model \ref{model:fair_rank}  \\
        $\Pi_2 \gets \Pi^*(2\hat{\bm{y}}^\top \bm{w} - \bm{y}^\top \bm{w})$ by Model \ref{model:fair_rank}   \\
        $\nabla {\mathcal L}(\bm{y}^\top \bm{w},  \hat{\bm{y}}^\top \bm{w}) \gets 
        \Pi_2 - \Pi_1 $\!\!\!\!\!\!\!\!\!\!\\
        \label{line:6}
        $\theta \gets \theta - \alpha \nabla {\mathcal L}(\bm{y}^\top \bm{w}, 
        \hat{\bm{y}}^\top \bm{y}) \frac{\partial \hat{\bm{y}}^\top \bm{w}}{\partial \theta}$
        \label{line:7}
    }
  }
\end{algorithm}

In the context of fair learning to rank, the goal is to predict the cost 
coefficients $\hat{\bm{y}}$ for Model \ref{model:fair_rank} which maximize the empirical DCG, equal to $\bm{y}^\top\, \Pi^*(\hat{\bm{y}})\, \bm{w}$
for  ground-truth relevance scores $\bm{y}$. A vectorized form can be written: 
\begin{equation}
    \bm{y}^\top \Pi \bm{w} = \overrightarrow{(\bm{y}^\top \bm{w})} \cdot \overrightarrow{\Pi},
\end{equation}
where $\overrightarrow{A}$ represents the row-major-order vectorization 
of a matrix $A$. 
Hence, the regret induced by prediction of cost coefficients $\hat{\bm{y}}$ is
\begin{equation}
    \cL(\bm{y}, \hat{\bm{y}}) = 
   \overrightarrow{(\bm{y}^\top \bm{w})} \cdot \overrightarrow{\Pi^*(\bm{y})}  
   -  \overrightarrow{(\bm{y}^\top \bm{w})} \cdot \overrightarrow{\Pi^*(\hat{\bm{y}})}.
\end{equation} 
Note that while the cost coefficients $\bm{y}$ can be predicted generically 
(i.e., predicting an $n^2$-sized matrix), 
the modeling approach taken in this paper is to predict item scores
 independently from individual feature vectors (resulting in an $n$-sized vector). 
 These values combine naturally with the known position bias values $\bm{v}$, to estimate DCG
 in the absence of true item scores. This simplification allows for
 learning independently over individual feature vectors, and was
 found in practice to outperform frameworks which use larger networks
  which take as input the entire feature vector lists. 

Algorithm \ref{alg:learning} maximizes the expected DCG of a learned 
ranking function by minimizing this regret. Its gradient is approximated as
\begin{equation}
    \overrightarrow{\Pi^*(2\hat{\bm{y}}^\top \bm{w} - \bm{y}^\top \bm{w})}  - 
     \overrightarrow{\Pi^*(\bm{y}^\top \bm{w})},
\end{equation}
with $\hat{\bm{y}}$ predicted as described in Section \ref{sec:predict}. 
To complete the calculation of gradients for the fair ranking model, 
the remaining chain rule factor of line \ref{line:7} 
is completed using the typical automatic differentiation.

\section{Multigroup Fairness}

SPOFR generalizes naturally to more than two groups. 
In contrast, multi-group fairness raises challenges for existing approaches that rely on penalty terms in the loss function \cite{yadav2021policy,singh2019policy,zehlike2020reducing}. 
Reference \cite{yadav2021policy} proposes to formulate multi-group fairness using the single constraint
\begin{equation}
\label{eq:fairness_one_group_mean_yadav}
    \sum_{g \neq g'}  \Big|\Big(\frac{1}{|G^{g}_q|}\mathbbm{1}_{G^{g}_q} 
                                - \frac{1}{|G^{g'}_q|}\mathbbm{1}_{G^{g'}_q}\Big)^\top 
              \; \Pi \; \bm{v} \Big| \leq \delta \;\;\;
\end{equation}
where $g$ and $g'$ are groups indicators in $\cG$, so that the average pairwise disparity between groups is constrained. However, this formulation suffers when $\delta \geq 0$, because the allowed fairness gap can be occupied by disparities associated with a single group in the worst case. Multiple constraints are required to provide true multi-group fairness guarantees and allow a controllable tradeoff between muti-group fairness and utility. Furthermore, the constraints~\eqref{eq:fairness_one_group} ensure satisfaction of \eqref{eq:fairness_one_group_mean_yadav} for appropriately chosen
$\delta$ and are thus a generalization of~\eqref{eq:fairness_one_group_mean_yadav}. If unequal group disparities are desired, $\delta$ may naturally be chosen differently for each group in the equations~\eqref{eq:fairness_one_group}.

\section{Experiments}
\label{sec:experiments}

This section evaluates the performance of SPOFR against the prior
approaches of \cite{yadav2021policy} and \cite{zehlike2020reducing}, 
the current state-of-the-art methods for fair learning rank, 
which are denoted by FULTR and DELTR respectively. The experimental evaluation follows the more realistic \emph{Partial 
Information setting} described in \cite{yadav2021policy}. A formal description of this setting is deferred to Appendix \ref{sec:partial_info}.

\subsubsection*{\bf Datasets}

Two full-information datasets were used in \cite{yadav2021policy} to generate partial-information  counterparts using click simulation:
\begin{itemize}[leftmargin=*, parsep=0pt, itemsep=2pt, topsep=2pt]
\item 
\emph{German Credit Data} is a dataset commonly used for studying algorithmic fairness. It contains information about $1000$ loan 
applicants, each described by a set of attributes and labeled as creditworthy 
or non-creditworthy. Two groups are defined by the binary feature A43, indicating the purpose of the loan applicant. 
The ratio of applicants between the two groups is around $8:2$.

\item
\emph{Microsoft Learn to Rank (MSLR)} is a standard benchmark dataset for LTR, containing a large number of queries from Bing with manually-judged relevance labels for retrieved web pages. The QualityScore attribute (feature id $133$) is used to define binary protected groups using the $40^{th}$ percentile as threshold as in \cite{yadav2021policy}. For evaluating fairness between $k > 2$ groups (multi-group fairness), $k$ evenly-spaced quantiles define the thresholds.
\end{itemize}

\noindent
Following the experimental settings of \cite
{yadav2021policy,singh2019policy}, all datasets are constructed to
contain item lists of size $20$ for each query. The German Credit and
MSLR training sets consist of $100$k and $120$k sample queries,
respectively while full-information test sets consist of $1500$ and
$4000$ samples. Additionally, as reported in Appendix \ref{app:results}, 
much smaller training sets result in analogous performance.

The reader is referred to \cite{yadav2021policy} for details of the
click simulation used to produce the training and validation sets.

\subsubsection*{\bf Models and Hyperparameters} The prediction of item
 scores is the same for each model, with a single neural network
 which acts at the level of individual feature vectors as described
 in Section \ref{sec:predict}. The size of each layer is half that of the previous, and the output is a scalar value representing an item score. Hyperparameters were selected as
 the best-performing on average among those listed in Table \ref
 {tab:hyperparams}, Appendix \ref{app:hyper}). {Final hyperparameters 
 for each model are as stated also in Table \ref{tab:hyperparams}, 
 and Adam optimizer is used in the production of each result.}

The special \emph{fairness parameters}, while also hyperparameters,
are treated differently. Recall that fair LTR systems often aim to
offer a tradeoff between utility and group fairness, so that fairness
can be reduced by an acceptable tolerance in exchange
for increased utility. For the baseline methods FULTR and DELTR, this
tradeoff is controlled indirectly through the constraint violation
penalty term denoted $\lambda$, as described in Section \ref
{section:learning_fair_ranking}. Higher values of $\lambda$
correspond to a preference for stronger adherence to fairness. In
order to achieve $\delta$-fairness for some specified $\delta$, many
values of $\lambda$ must be searched until a trained model satisfying
$\delta$-fairness is found. As described in Section \ref
{section:learning_fair_ranking}, this approach is unwieldly. In the
case of SPOFR, the acceptable violation magnitude $\delta$ can be
directly specified as in Definition (\ref{def:delta_fairness}).

\medskip\noindent
The performance of each method is reported on the full-information test set for which all relevance labels are known.
 Ranking utility and fairness are measured with average DCG 
 (Equation ~\eqref{eq:policy_utility}) and fairness violation (Equation~\eqref
 {eq:fairness_violation}), where each metric is computed on average
 over the entire dataset. We set the position bias power $p=1$, so that 
 $v_j = \nicefrac{1}{(1+j)}$ when computing fairness disparity.

\begin{figure}[!tb]
\centering
\includegraphics[width=0.498\linewidth]{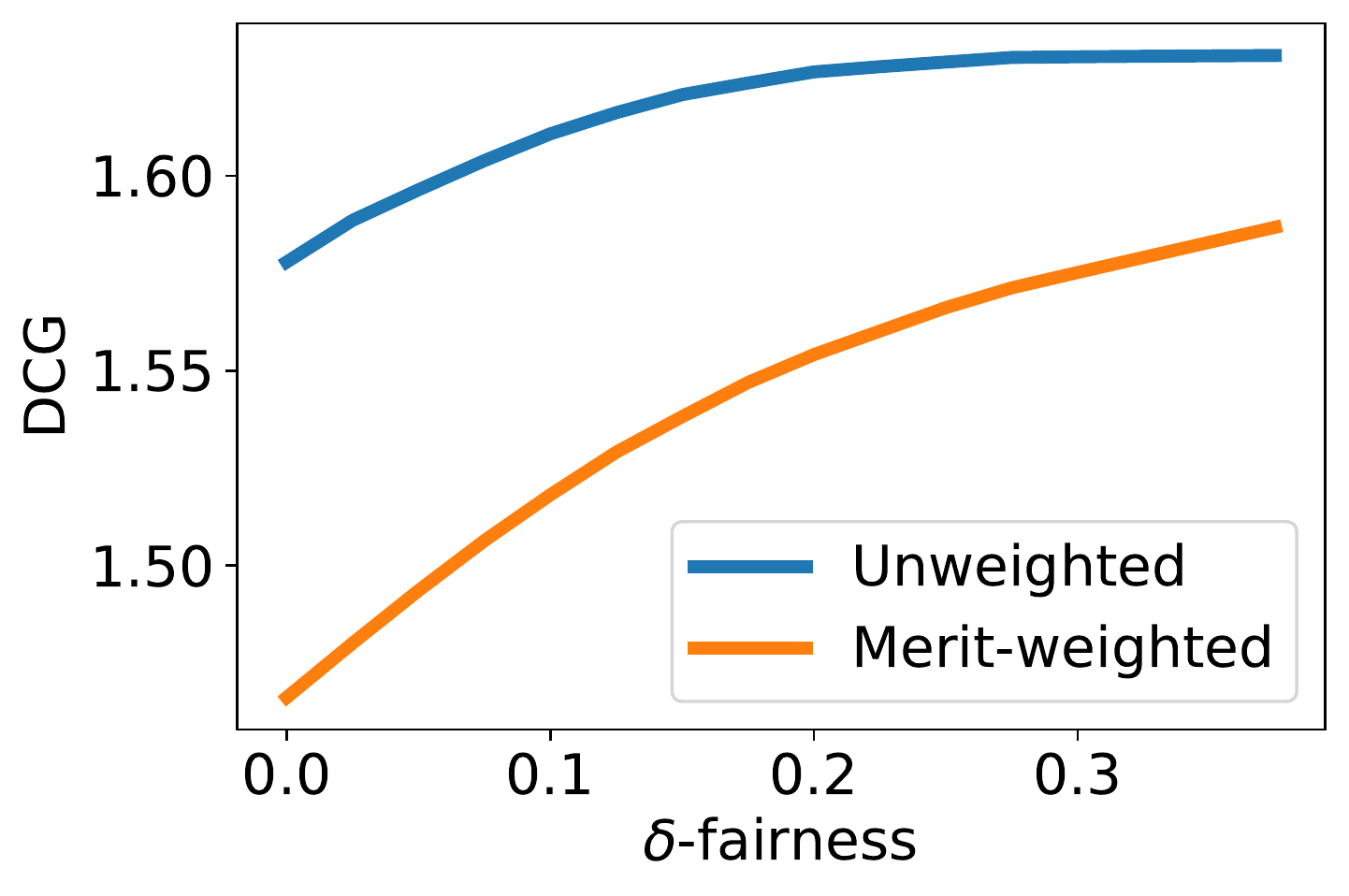}\!\!\!
\includegraphics[width=0.498\linewidth]{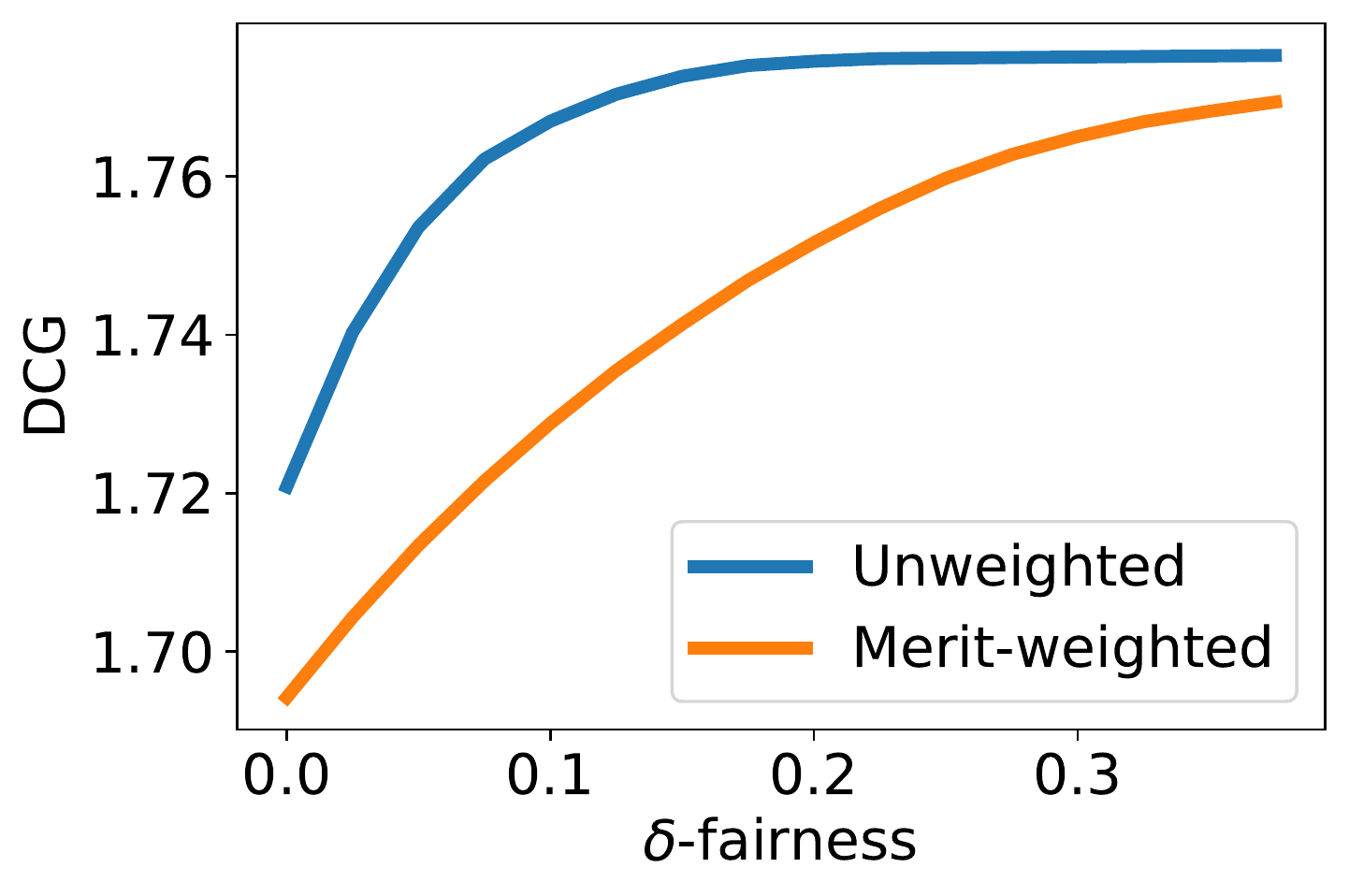} 
\caption{Fairness-Utility tradeoff for unweighted and merit-weighted fairness
on credit (left) and MSLR (right) datasets.}
\label{fig:spo_dcg_fairness_disptype}
\end{figure}

\subsubsection*{\bf Fairness-Utility Tradeoff for Two Groups}
The analysis first focuses on experiments involving two protected groups. 
Figure \ref{fig:spo_dcg_fairness_disptype} shows the average test
DCG attained by SPOFR on both the German Credit and MSLR datasets,
for each level of both unweighted and merit-weighted
$\delta$-fairness as input to the model.  Each result comes from a
model trained with final hyperparameters as shown in Table \ref{tab:hyperparams}. 
Recall that each value of $\delta$ (defined as in Definition 
\ref{def:delta_fairness}) on the x-axis is guaranteed to bound the 
ranking policy's expected fairness violation in response to \emph{each} query. Note the clear trend showing an increase in utility with the relaxation of the fairness bound $\delta$, for all metrics and datasets. Note also that, in the datasets studied here, average merit favors 
the majority group. Merit-weighted group fairness can thus constrain 
the ranking positions of the minority group items further down than 
in unweighted fairness, regardless of their individual relevance scores, 
leading to more restricted (thus with lower utility) policies than in the unweighted case. 

In addition to results over large training sets, reported in the figure, Appendix \ref{app:additional_experiments} reports results for the model trained on datasets with $5K$ and $50k$ samples for German Credit, and $12K$ and $36K$ samples, for MSLR dataset. These results show the same trends as those reported above.

\begin{figure}[!tb]
\centering
\includegraphics[width=0.6\linewidth]{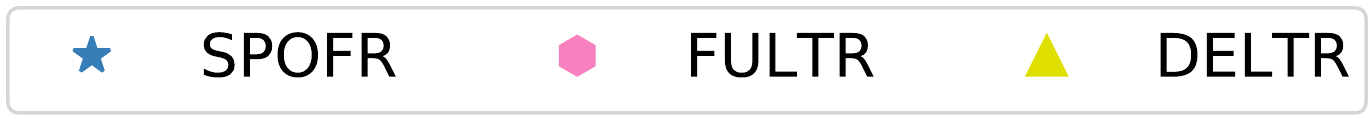}
\includegraphics[width=0.495\linewidth]{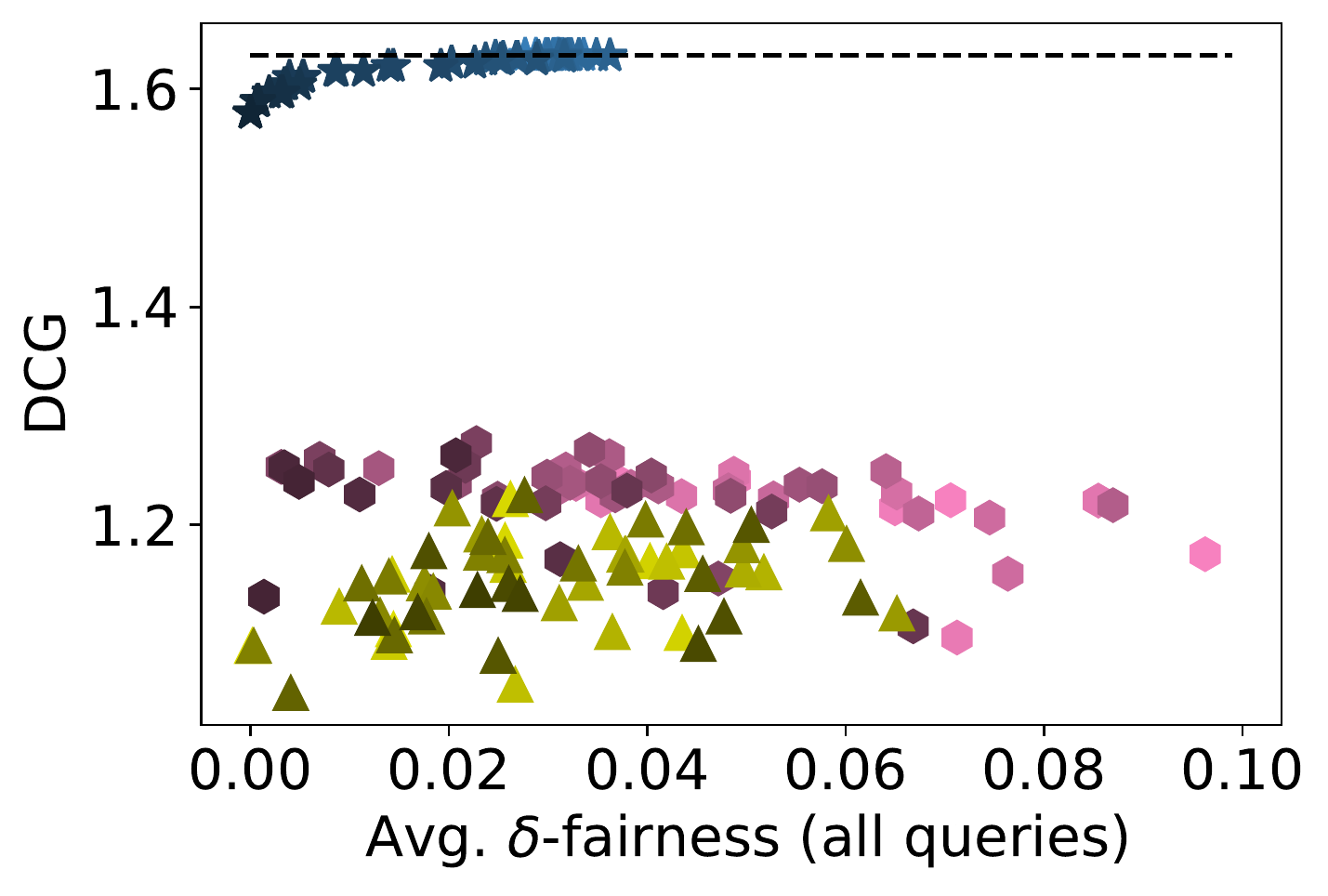} 
\includegraphics[width=0.495\linewidth]{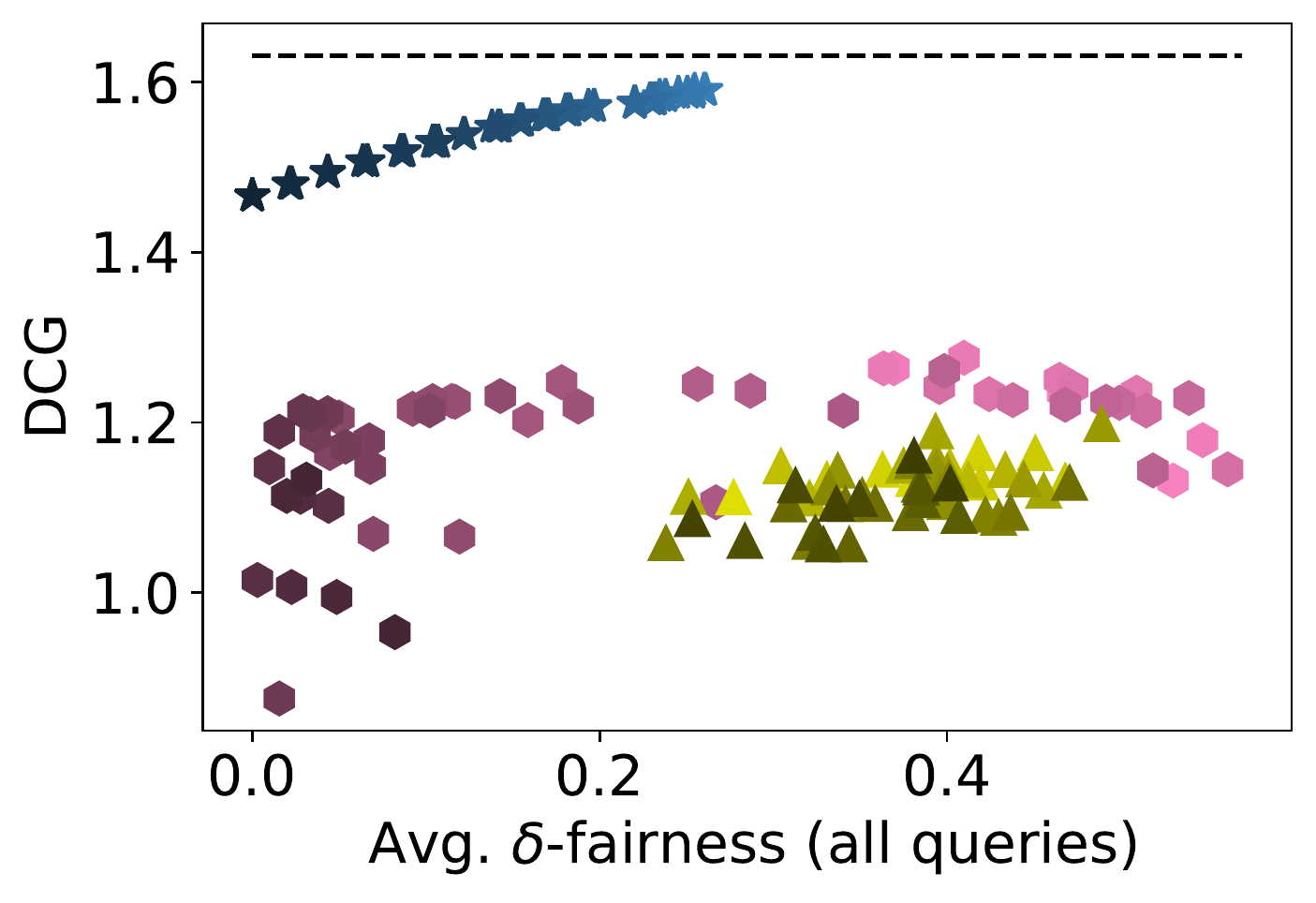} \\
\includegraphics[width=0.495\linewidth]{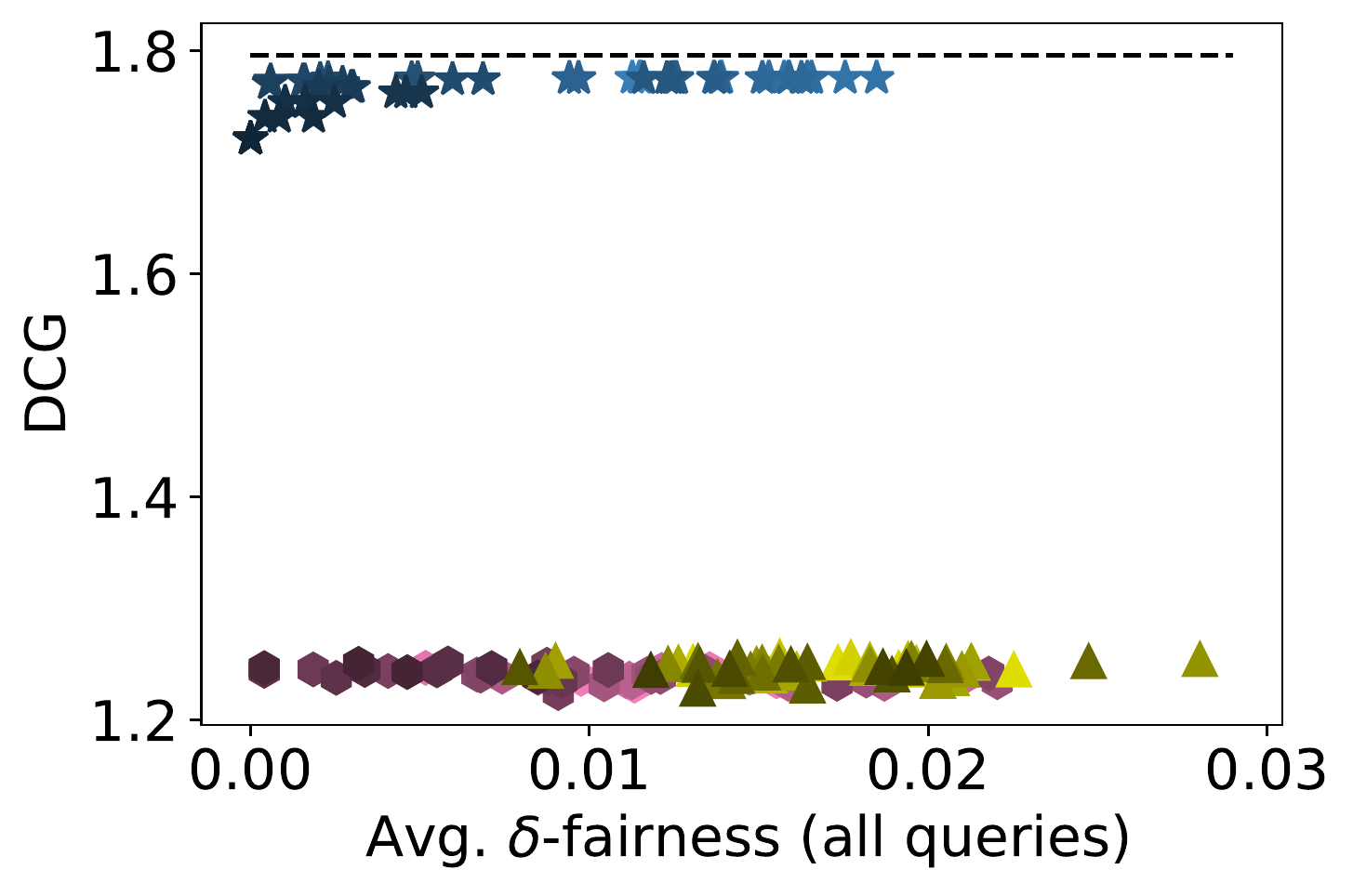} 
\includegraphics[width=0.495\linewidth]{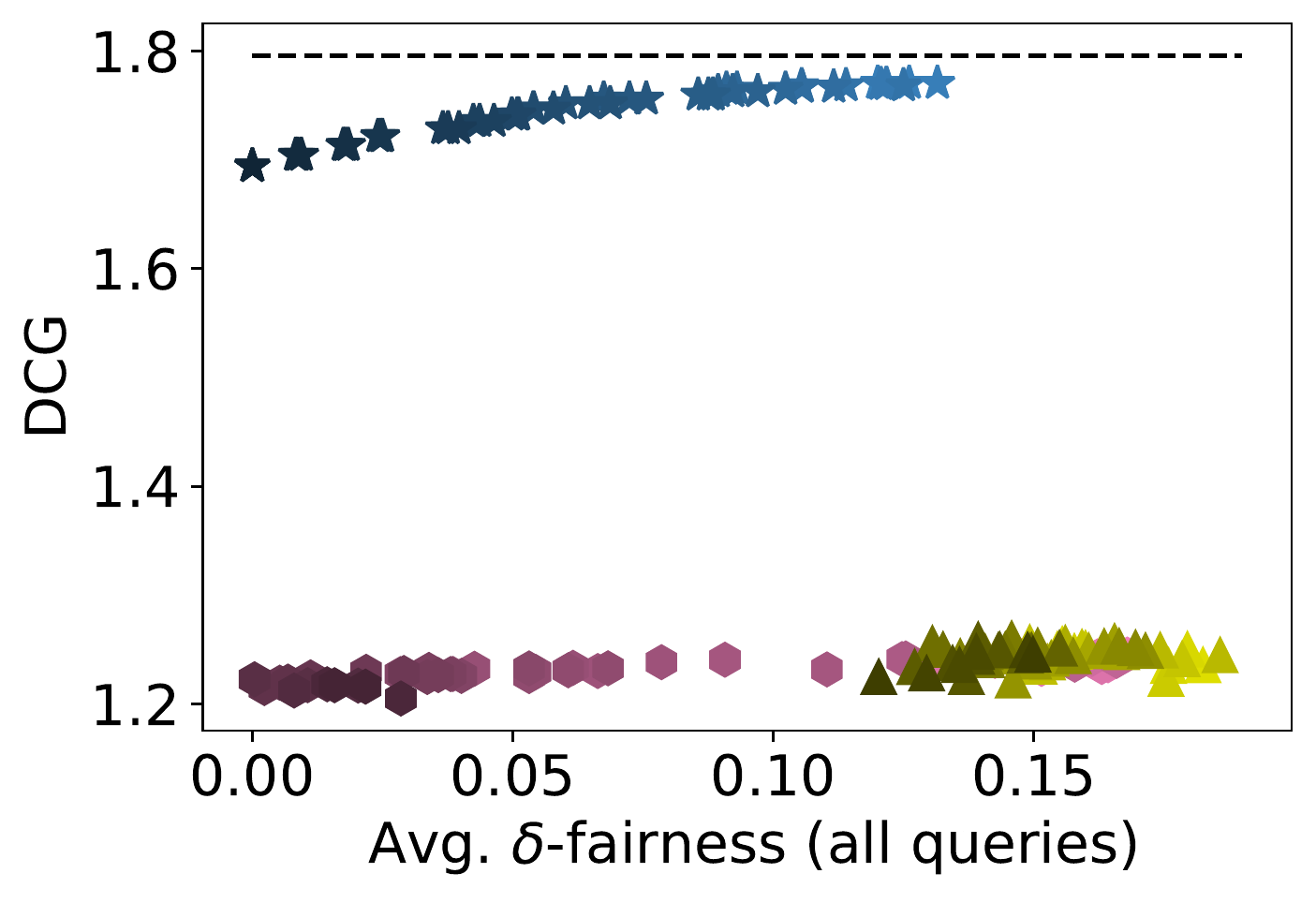} 
\caption{Fairness-Utility tradeoff for unweighted (left) and merit-weighted (right) fairness 
         on credit (top) and MSLR (bottom) datasets.}\label{fig:grid_searches}
\end{figure}
\subsubsection*{\bf Fairness Parameter Search} Figure \ref
 {fig:grid_searches} shows the average DCG vs average fairness
 disparity over the \emph{test set} due to SPOFR, and compares it with 
 those attained by FULTR and DELTR. 
 Each point represents the performance of a single trained
 model, taken from a grid-search over \emph{fairness parameters}
 $\delta$ (for SPOFR) and $\lambda$ (for FULTR and DELTR) between the
 minimum and maximum values in Table \ref{tab:hyperparams}. Darker
 colors represent more restrictive fairness parameters in each case.
 Non-fairness hyperparameters take on the final values shown also in
 Table \ref{tab:hyperparams}, and each model specification is
 repeated with $3$ random seeds. 
 Note that points on the grid which are lower on the x-axis and higher 
 on the y-axis represent results which are strictly superior relative 
 to others, as they represent a larger utility for smaller fairness 
 violations. 
 Dashed lines represent the maximum utility attainable in each case, computed by averaging  over the test set the maximum possible DCG associated to each relevance vector.

 First, we observe that the expected fairness violations due to SPOFR 
 are much lower than the fairness levels guaranteed by the listed 
 fairness parameters $\delta$. This is because $\delta$ is a bound 
 on the worst-case violation of fairness associated with any query, but 
 actual resulting fairness disparities are typically much lower on 
 average. 

 Second, the figure shows that SPOFR attains a substantial 
 improvement in utility over the baselines, while exhibiting more 
 consistent results across independently trained models. Note the dashed line represents the theoretical maximum attainable utility; remarkably, as the fairness parameter is relaxed, the DCG attained by SPOFR converges very close to this value.
 Section \ref{sec:theory} provides theoretical motivation to explain 
 these marked improvements in performance.

 Finally, notice that for FULTR and DELTR, large $\lambda$ values 
 (darker colors) should be associated with smaller fairness violations,
 compared to models trained with smaller $\lambda$ values (lighter colors).
 However, this trend is not consistently observable: These results show 
 the challenge to attain a meaningful relationship between the fairness 
 penalizers $\lambda$ and the fairness violations in these state-of-the-art
 fair LTR methods. A similar observation also pertains to utility; It is expected
 that more permissive models in terms of fairness would attain larger 
 utilities; this trend is not consistent in the FULTR and DELTR models.
 In contrast, the ability of the models learned by SPOFR to \emph{guarantee}
 satisfying the desired fairness violation equip the resulting LTR models 
 with much more interpretable and consistent outcomes.

\begin{figure*}[t]
\centering
\includegraphics[width=0.95\linewidth]{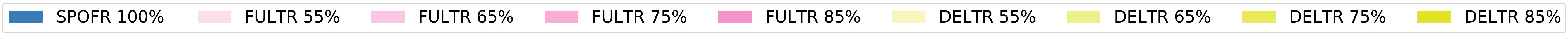}\\
\includegraphics[width=0.24\linewidth]{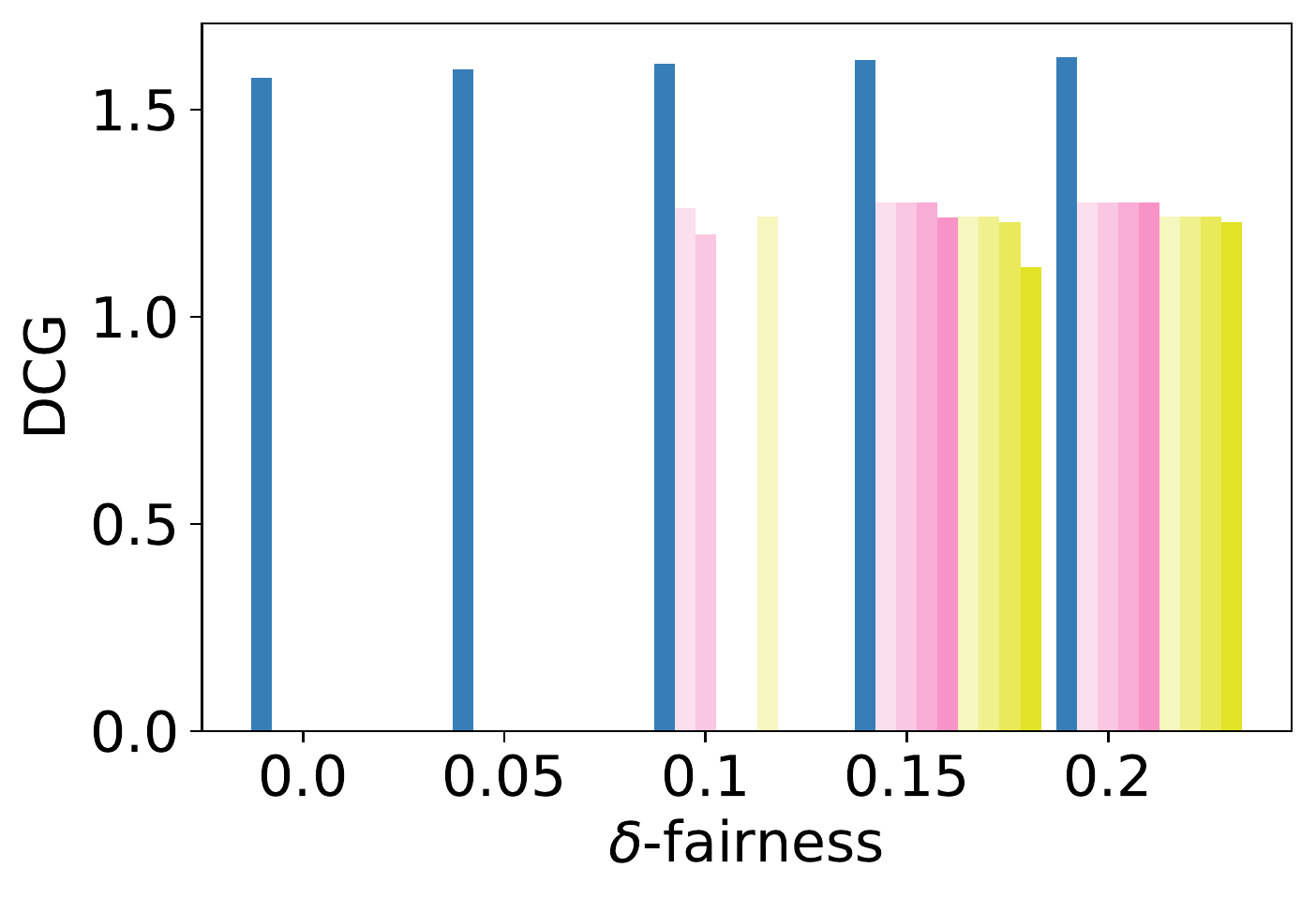}
\includegraphics[width=0.24\linewidth]{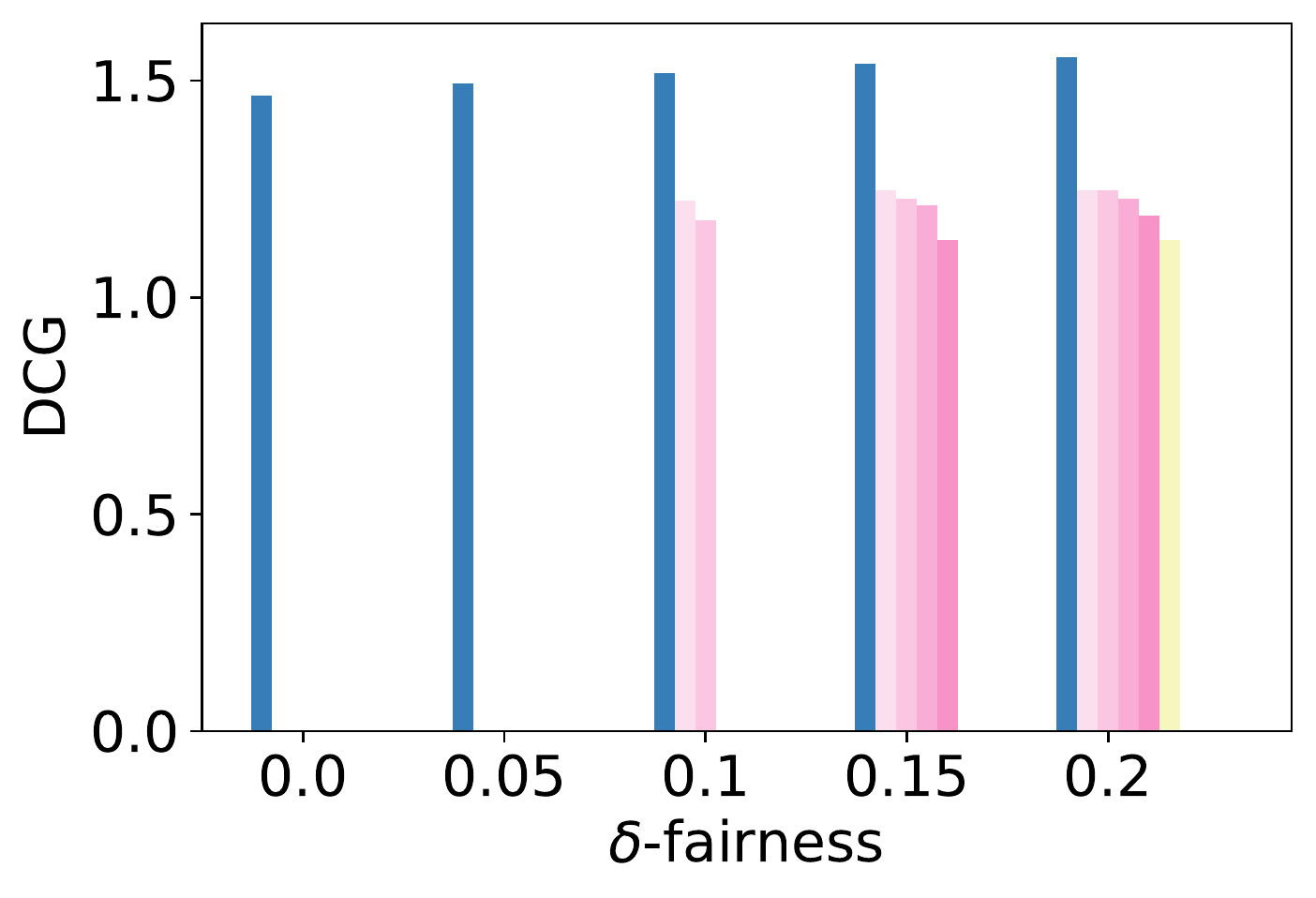}
\includegraphics[width=0.24\linewidth]{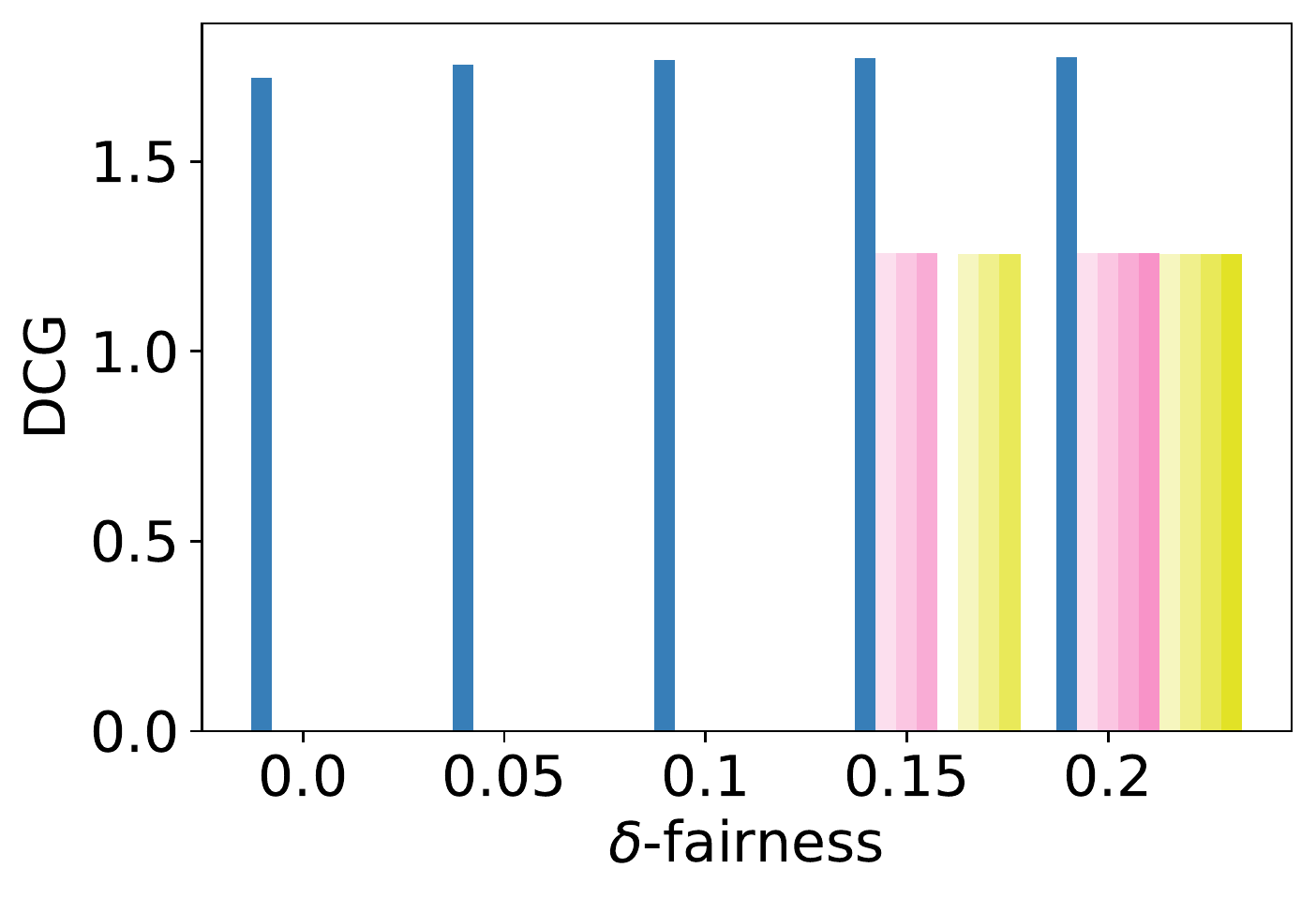}
\includegraphics[width=0.24\linewidth]{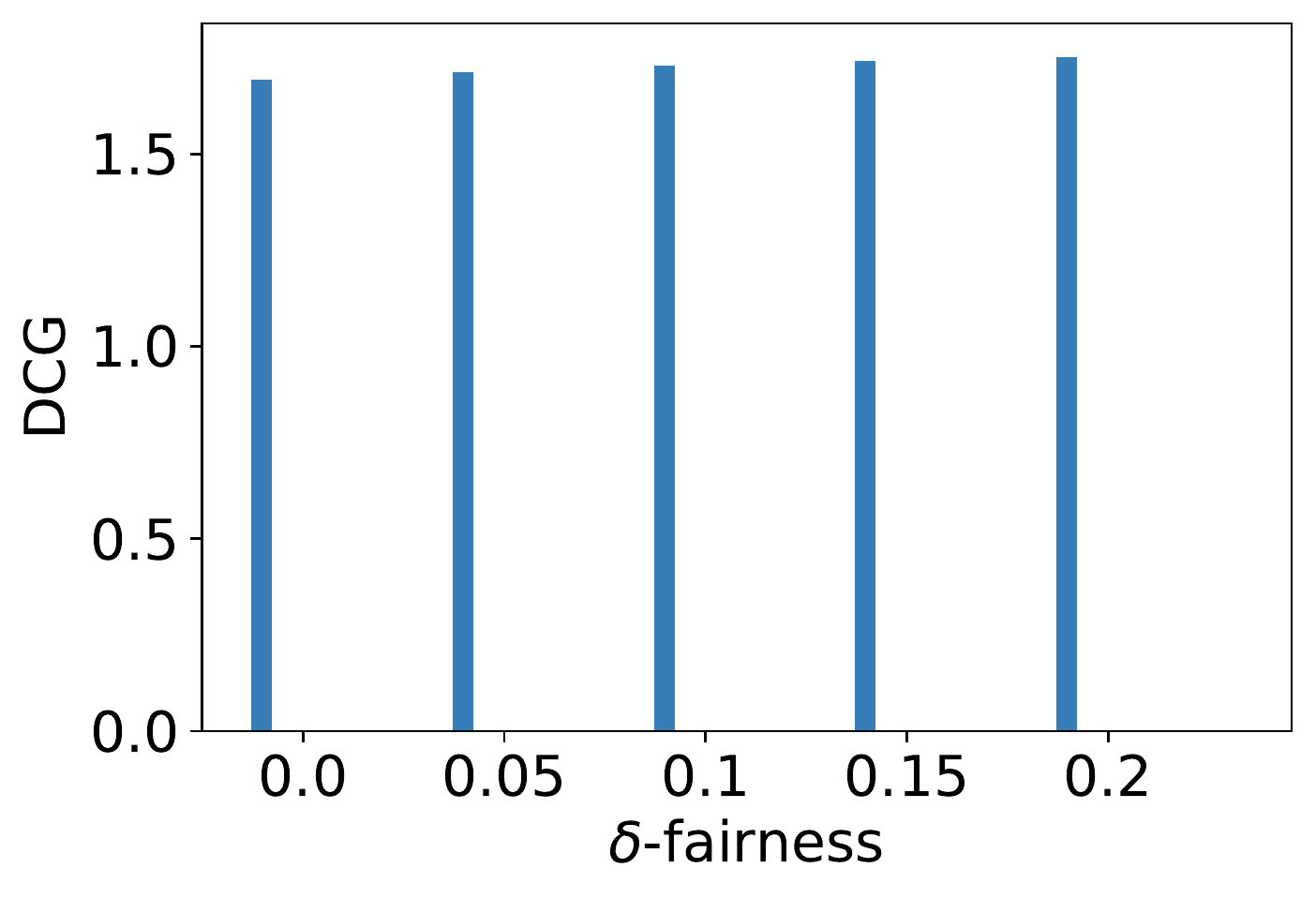}
\caption{Query level guarantees: German credit unweighted ($1^{\text{st}}$ column)
                    and merit-weighted fairness ($2^{\text{nd}}$ column);  
                    MSLR unweighted ($3^{\text{rd}}$ column) and 
                    merit-weighted fairness ($4^{\text{th}}$ column).}
\label{fig:query_level}
\end{figure*}

\subsubsection*{\bf Query-level Guarantees} 
As discussed in Section \ref{section:learning_fair_ranking}, 
current fair LTR methods apply a fairness violation term on 
average over all training samples. Thus, disparities in favor of one group
can cancel out those in favor of another group leading to \emph{individual} 
policies that may not satisfy a desired fairness level.
This section illustrates on these behaviors and analyzes the fairness 
guarantees attained by each model compared at the level of each 
individual query. 

The results are summarized in Figure \ref{fig:query_level} which compares, 
SPOFR with FULTR (top) and SPOFR with DELTR (bottom).
 Each bar represents the maximum expected DCG attained by a LTR model
 while guaranteeing $\delta$-fairness at the query level for
 $\delta$ as shown on the x-axis. Since neither baseline method can
 satisfy $\delta$-fairness for every query, confidence levels are
 shown which correspond to the percentage of queries within the test
 set that resulted in ranking policies that satisfy $delta$-fairness.
 If no bar is shown at some fairness level, it was satisfied by no
 model at the given confidence level. 
 Results were drawn from the same set of models analyzed in the previous section. 
 
 Notably, SPOFR satisfies $\delta$-fairness with $100$ percent confidence
 while also surpassing the baseline methods in terms of expected
 utility. This is remarkable and is due partly to the fact that the 
 baseline methods can only be specified to optimize for fairness \emph{on average}
 over all queries, which accomodates large query-level fairness
 disparities when they are balanced in opposite directions; i.e., in
 favor of opposite groups. In contrast SPOFR guarantees the specified 
 fairness violation to be attained for ranking policies associated with
 each individual query.

\subsubsection*{\bf Multi-group Fairness} Finally, Figure \ref
 {fig:multigroup} shows the fairness-utility tradeoff curves attained
 by SPOFR for each number of groups between $2$ and $7$ on the MSLR
 dataset. 
 Note the decrease in expected DCG as the number of
 groups increases. This is not necessarily due to a degradation in
 predictive capability from SPOFR; the expected utility of any
 ranking policy necessarily decreases as fairness constraints are
 added. In fact, the expected utility converges for each multi-group
 model as the allowed fairness gap increases. Because this strict
 notion of multi-group fairness in LTR is uniquely possible using
 SPOFR, no results from prior approaches are available for direct
 comparison.

\begin{figure}
\includegraphics[width=0.65\linewidth]{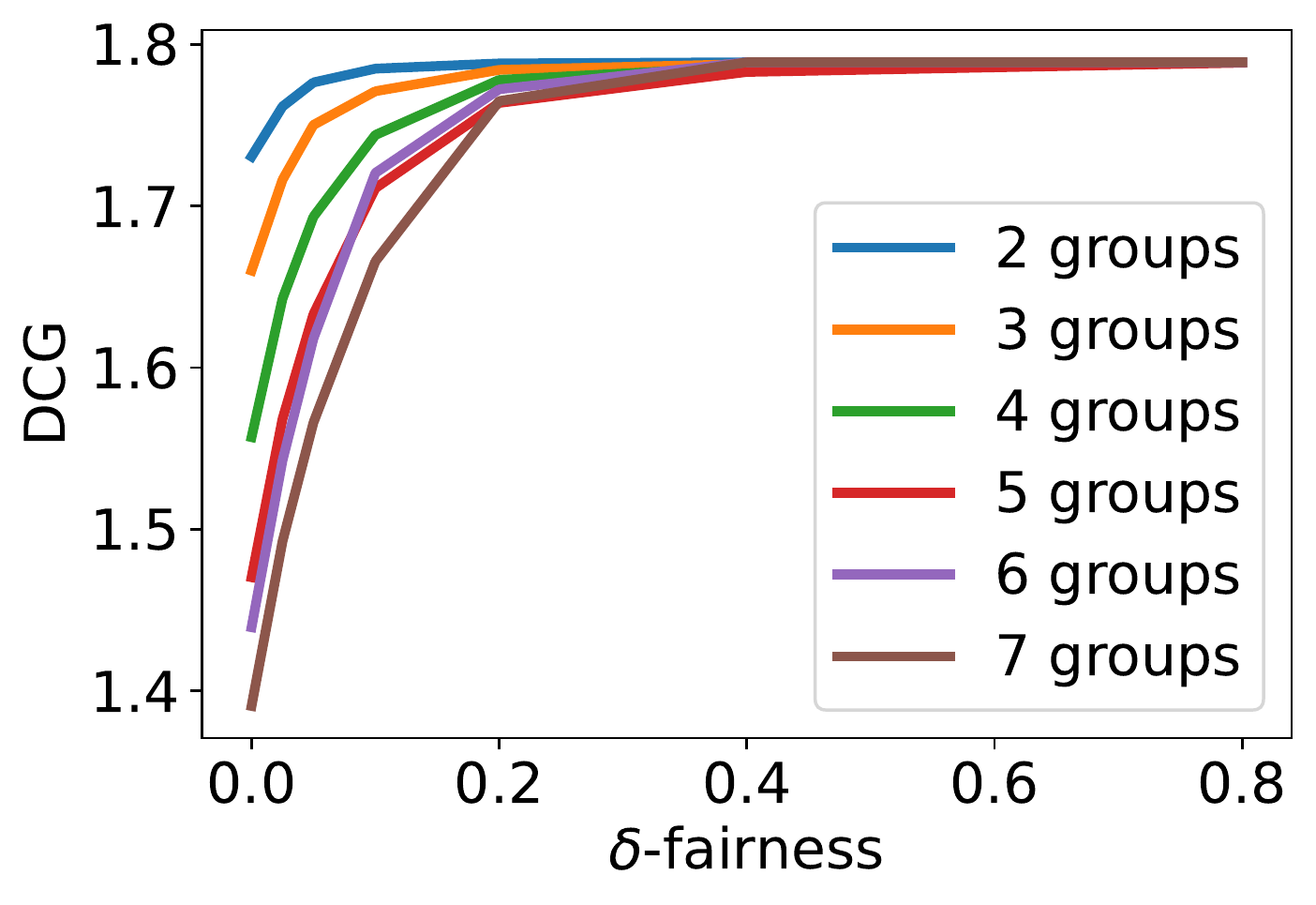} 
\caption{Multigroup Fairness on MSLR 120k.}
\label{fig:multigroup}
\end{figure}

\section{Discussion}
\label{sec:theory}

\subsubsection*{\bf Theoretical Remarks}
This section provides theoretical intuitions to explain the strong
performance of SPOFR. As direct outputs of a linear
programming solver, the ranking policy returned by SPOFR are 
subject to the properties of LP optimal solutions. This allows for
certain insights on the representative capacity of the ranking
model and on the properties of its resulting ranking policies. 
Let predicted scores be said to be \emph
{regret-optimal} if their resulting policy induces zero regret, i.e, $\bm{y}^\top \Pi^*(\bm{y}) \; \bm{v} - \bm
{y}^\top \Pi^*(\hat{\bm{y}})\; \bm{w} = 0$. That is equivalent to
the maximization of the empirical utility. 

\begin{theorem}[Optimal Policy Prediction]
\label{prop:LP_opt_policy} For any given ground-truth relevance scores
 $\bm{y}$, there exist predicted item scores which maximize the empirical utility relative to $\bm{y}$.
\end{theorem}

\begin{proof}
\label{proof:LP_opt_policy} It suffices to predict $\hat{\bm{y}} = \bm
 {y}$. These scores are regret-optimal by definition, thus maximizing
 empirical utility. 
\end{proof}
Note that the above property is due to the alignment between the
structured policy prediction of SPOFR and the evaluation
metrics, and is not shared by prior fair learning to rank
frameworks. 

Next, recall that any linear programming problem has a finite number
of feasible solutions whose objective values are distinct 
\cite{bazaraa2008linear}. {\em There is thus not
only a single point but a region of $\hat{\bm{y}}$ which are
regret-optimal under $\bm{y}$, with respect to any instance of Model 
\ref{model:fair_rank}}. This important property eases the difficulty in finding
model parameters which maximize the empirical utility for any input
sample, as item scores \emph{do not} need to be predicted precisely
in order to do so.

Finally, the set of $\hat{\bm{y}}$ which minimize the regret with respect to
any instance of Model \ref{model:fair_rank} overlaps (has nonempty
intersection) with the set of $\hat{\bm{y}}$ which minimize the regret with respect to any other instance of the model, regardless of fairness 
constraints, under the same ground-truth $\bm{y}$. To show this, it suffices to exhibit a value of $\hat{\bm{y}}$ which minimizes the respective regret
in every possible instance of Model~\ref{model:fair_rank}, namely
$\bm{y}$:
\begin{equation}
    \bm{y}^\top \Pi_{f_1}^*(\bm{y}) \bm{w}  -  \bm{y}^\top \Pi_{f_1}^*(\bm{y}) \bm{w} 
    = 0 =  
    \bm{y}^\top \Pi_{f_2}^*(\bm{y}) \bm{w}  -  \bm{y}^\top \Pi_{f_2}^*(\bm{y}) \bm{w},
\end{equation}
where $\Pi_{f_1}^*(y)$ and $\Pi_{f_2}^*(y)$ are the optimal policies
subject to distinct fairness constraints $f_1$ and $f_2$. This
implies that a model which learns to rank fairly under this framework
need not account for the group composition of item lists in order to
maximize empircal utility; {\em It suffices to learn item scores from
independent feature vectors, rather than learn the higher-level
semantics of feature vector lists required to enforce fairness}. This
is because group fairness is ensured automatically by the embedded
optimization model. This aspect is evident in Figure~\ref{fig:grid_searches}
showing that SPOFR exhibits small losses in utility even in restrictive 
fairness settings. 

The empirical results presented in the previous section, additionally, 
show that the utility attained by SPOFR is close to optimal on the 
test cases analyzed. 
Together with the theoretical observations above, this suggests that, on 
the test cases analyzed, fairness does not change drastically the objective 
of the optimal ranking policy. This may be an artifact of the LTR tasks, obtained from \cite{yadav2021policy}, being relatively easy predict given a sufficiently powerful model. These observation may signal a need for the study and curation of more challenging benchmark datasets for future research on fairness in learning to rank.

\subsubsection*{\bf SPOFR Limitations}
The primary disadvantage of SPOFR is that it cannot be expected to learn to rank lists of arbitrary size, as runtime increases with the size of the lists to be ranked. In contrast to penalty-based methods, which require a single linear pass 
to the neural network to derive a ranking policy for a given query, 
SPOFR requires solving a linear programming problem to attain an optimal 
ranking policy. While this is inevitably computationally more expensive,
solving the LP of Model \ref{model:fair_rank} requires low degree 
polynomial time in the number of items to rank \cite{van2020deterministic},
due to the sparsity of its constraints. Fortunately, this issue can be vastly alleviated with the application of hot-starting schemes \cite{mandi2019smart}, since the SPO framework relies on iteratively updating a stored solution to each LP instance for slightly different objective coefficients as model weights are updated. \emph{Thus, each instance of Model \ref{model:fair_rank} need not be solved from scratch.}
Appendix \ref{app:efficiency} reports a detailed discussion on the steps taken in this work to render the proposed model both computationally and memory efficient.

\section{Conclusions}

This paper has described SPOFR, a framework for learning fair ranking functions by integrating constrained optimization with deep learning. By enforcing fairness constraints on its ranking policies at the level of each prediction, this approach provides direct control over the allowed disparity between groups, which is guaranteed to hold for every user query. Since the framework naturally accommodates the imposition of many constraints, it generalizes to multigroup fair LTR settings without substantial degradation in performance, while allowing for stronger notions of multigroup fairness than previously possible. Further, it has been shown to outperform previous approaches in terms of both the expected fairness and utility of its learned ranking policies. By integrating constrained optimization algorithms into its fair ranking function, SPOFR allows for analytical representation of expected utility metrics and end-to-end training for their optimization, along with theoretical insights into properties of its learned representations. These advantages may highlight the integration of constrained optimization and machine learning techniques as a promising avenue to address further modeling challenges in future research on learning to rank.

\begin{acks}
This research is partially supported by NSF grant 2007164.
\end{acks}

\bibliographystyle{abbrvnat}
\bibliography{bib}

\pagebreak\newpage
\appendix

\begin{figure}[tbh]
\includegraphics[width=0.490\linewidth]{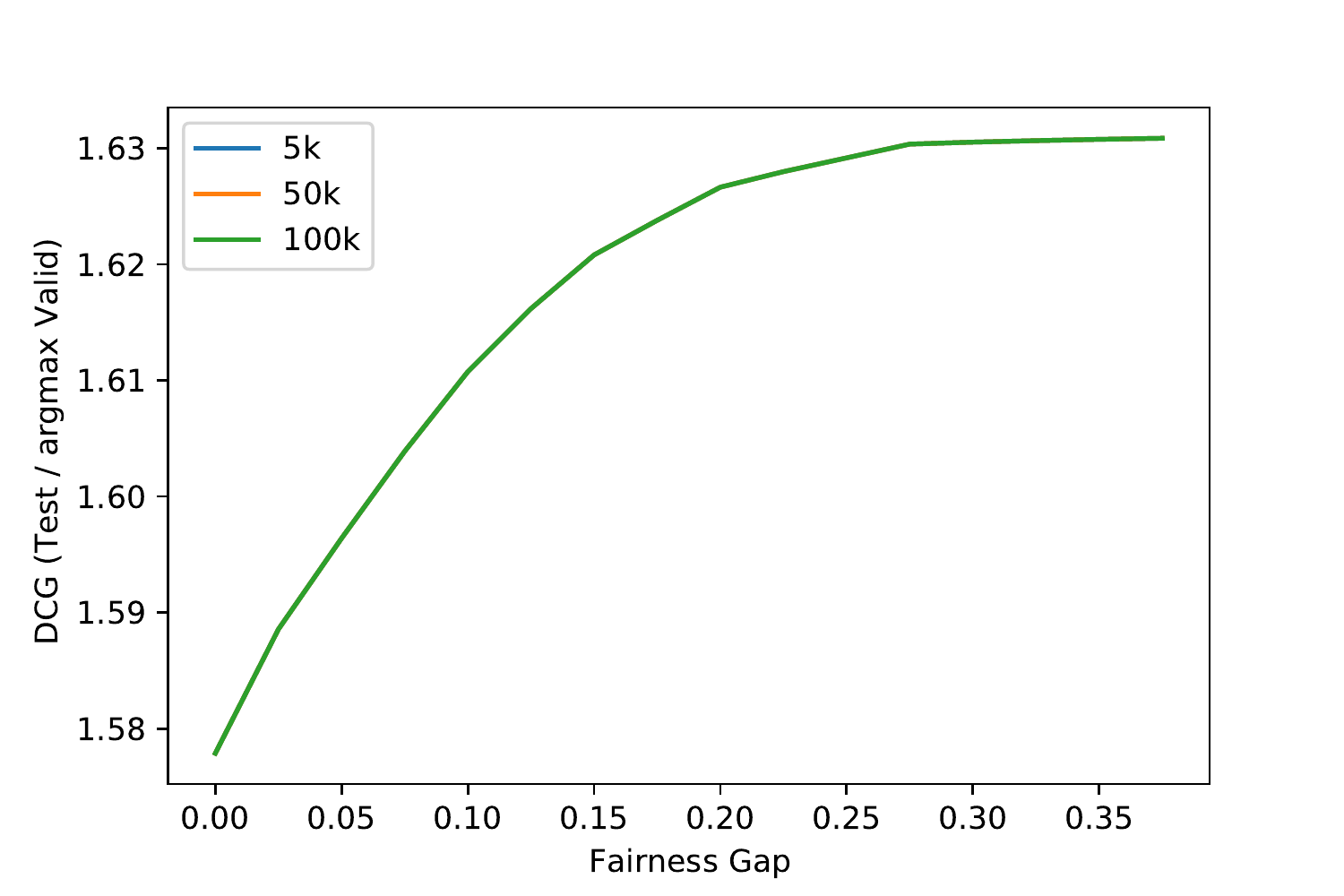} 
\includegraphics[width=0.490\linewidth]{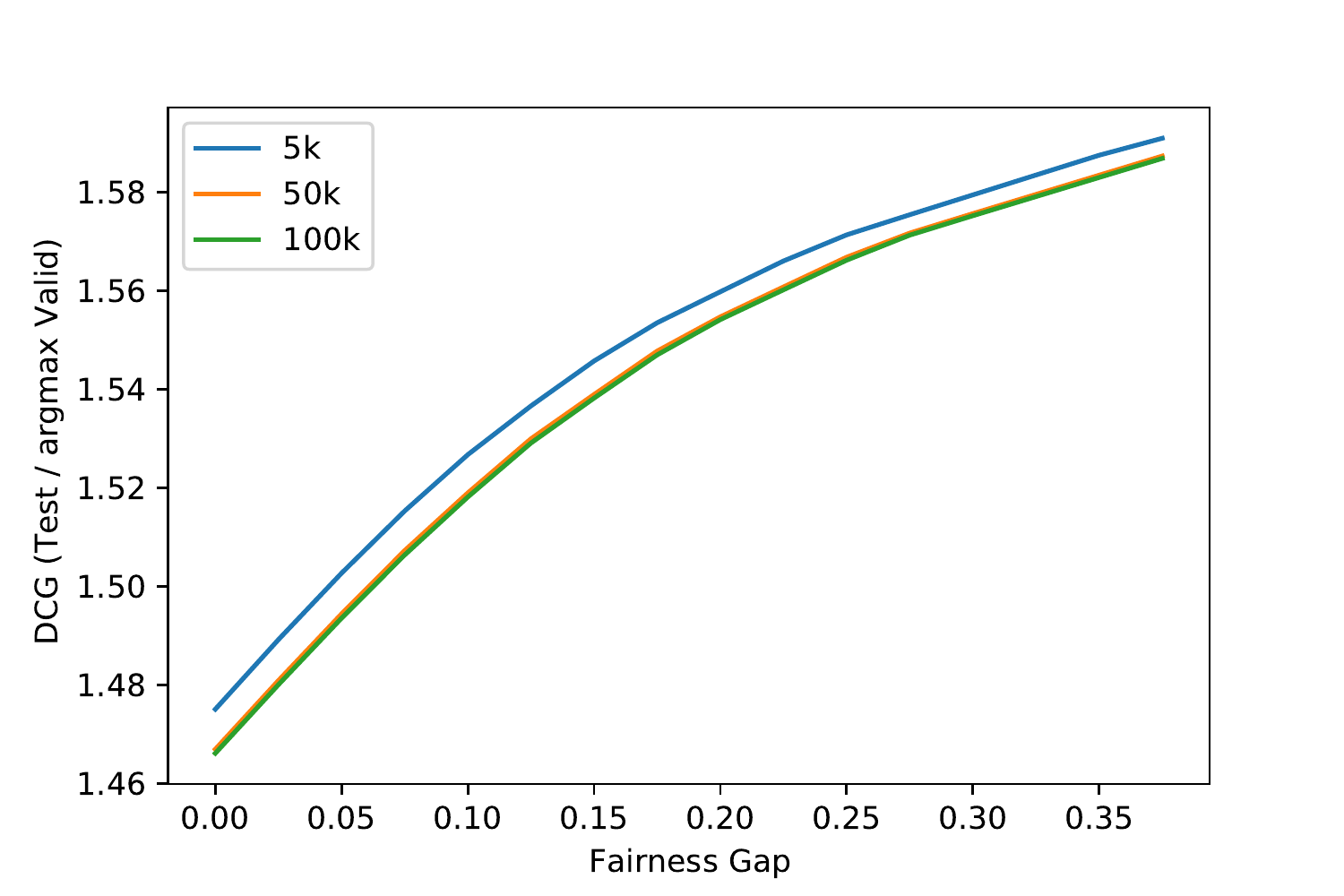} \\
\includegraphics[width=0.490\linewidth]{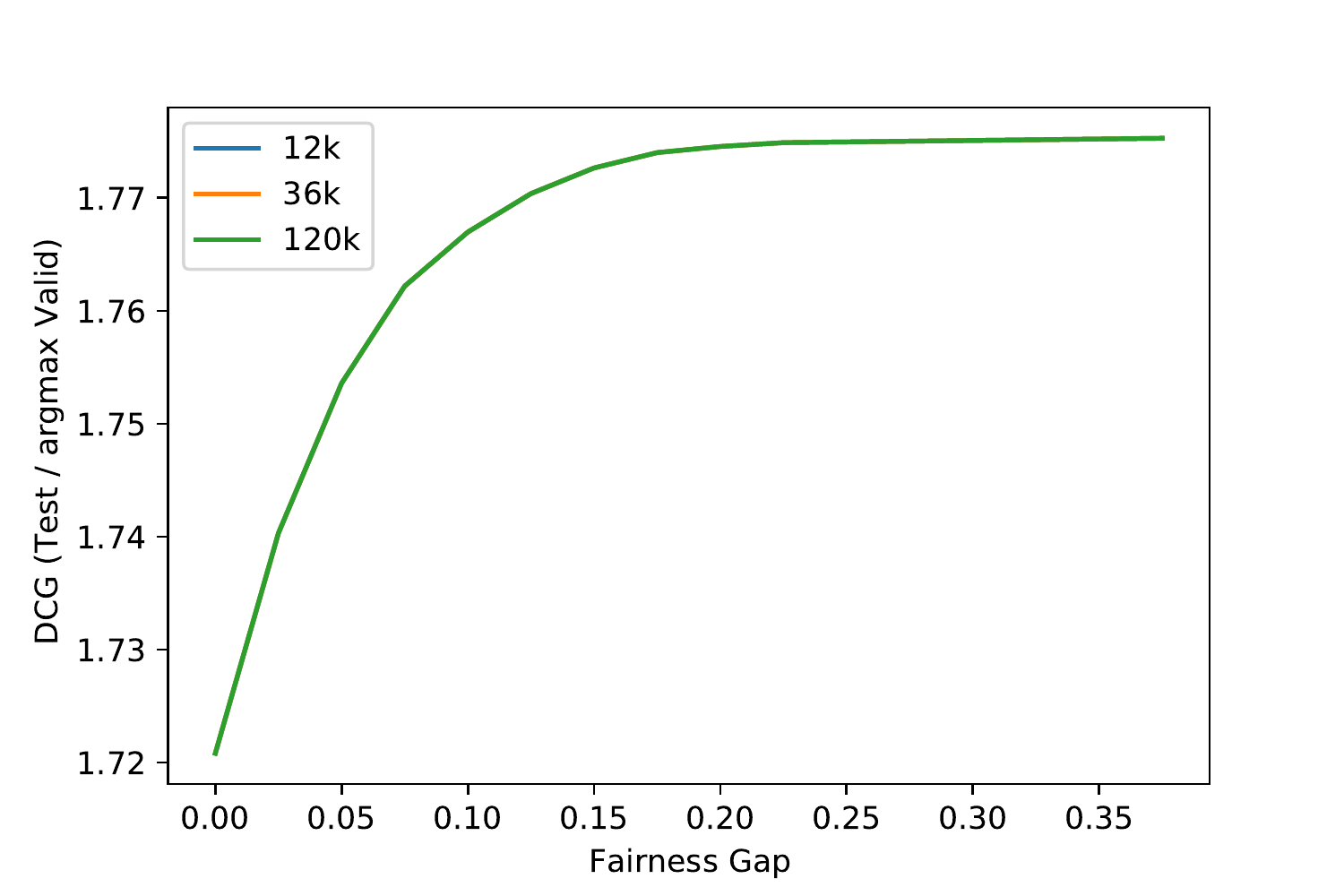} 
\includegraphics[width=0.490\linewidth]{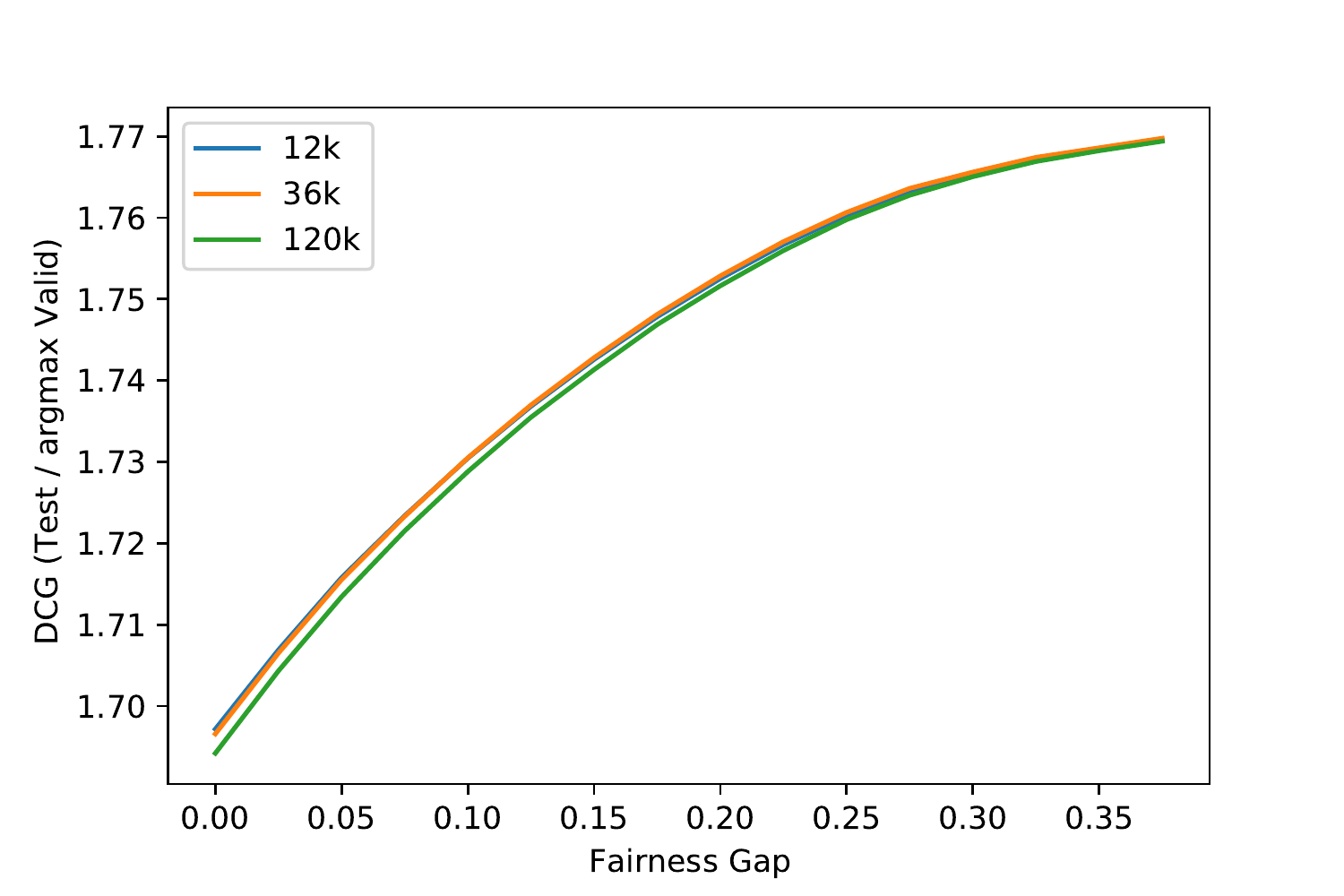} 
\caption{Top left: German Credit Unweighted Fairness;  Top right: German Credit Merit-Weighted Fairness;  Bottom left: MSLR Unweighted Fairness;  Bottom right: MSLR Merit-Weighted Fairness. Average DCG vs allowed fairness gap, SPOFR on datasets of $3$ different sizes}
\label{fig:spo_dcg_fairness_sizes}
\end{figure}
\begin{figure*}[!t]
\centering
\includegraphics[width=0.9\linewidth]{legend_conf.pdf}\\
\includegraphics[width=0.24\linewidth]{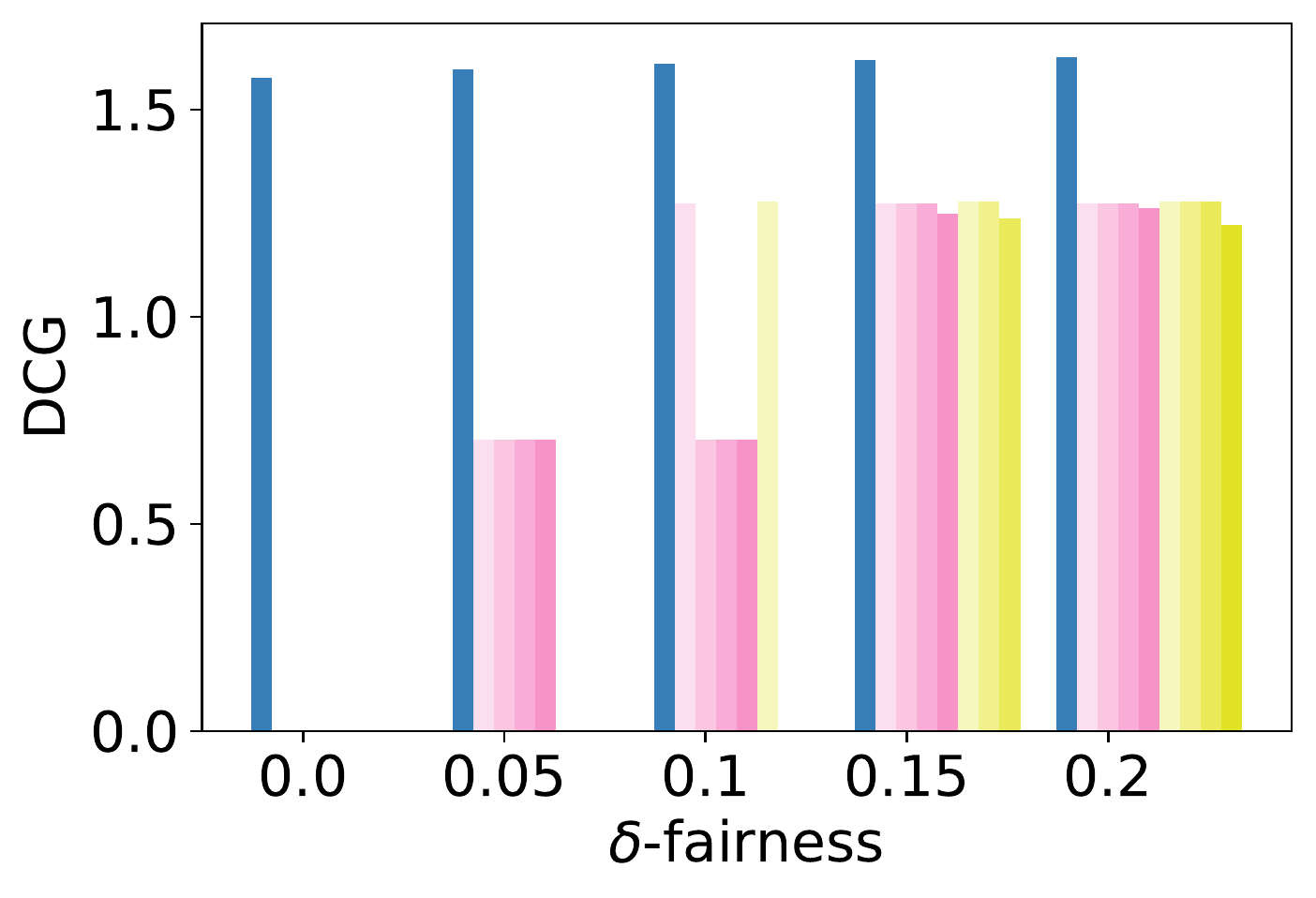}
\includegraphics[width=0.24\linewidth]{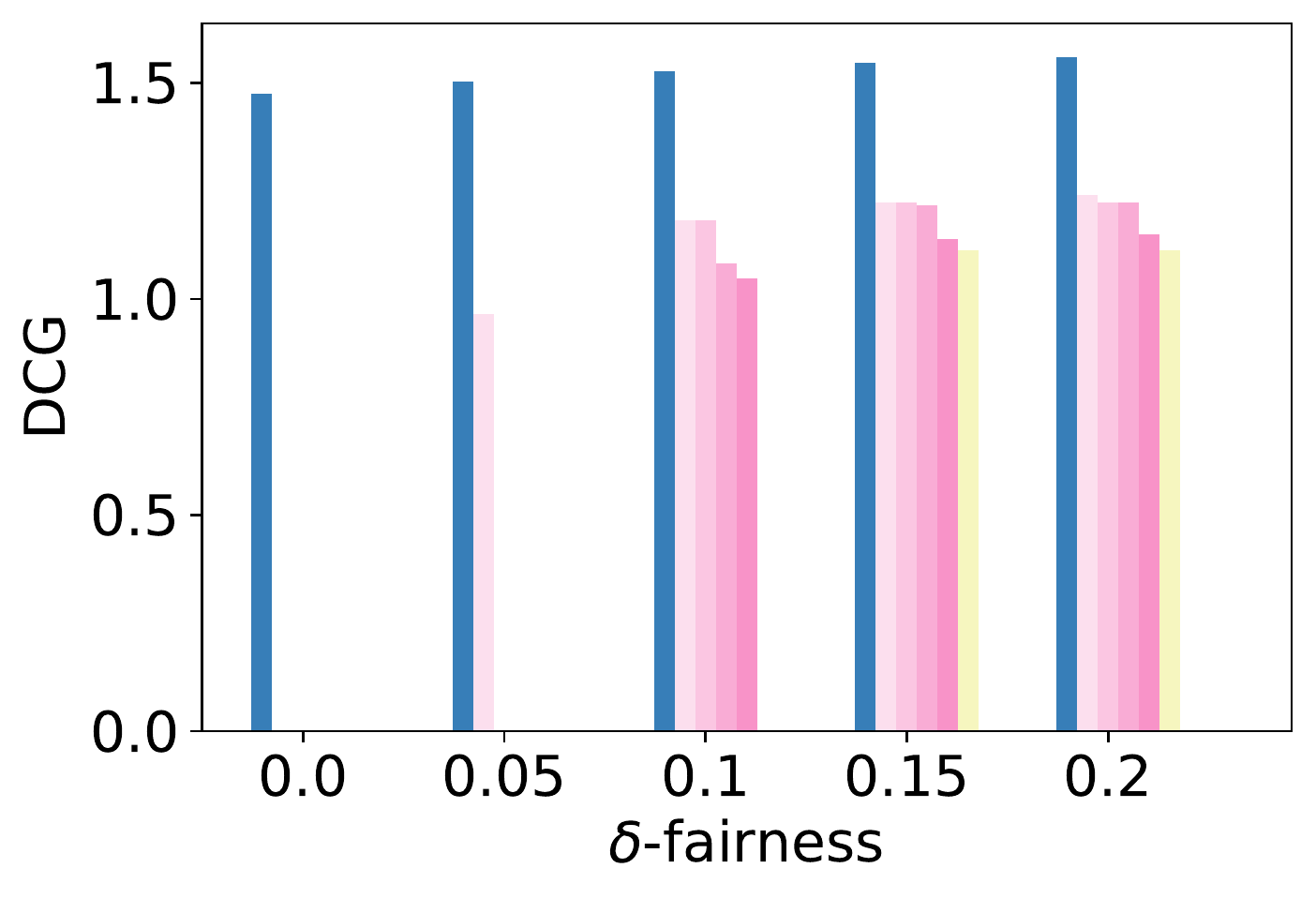}
\includegraphics[width=0.24\linewidth]{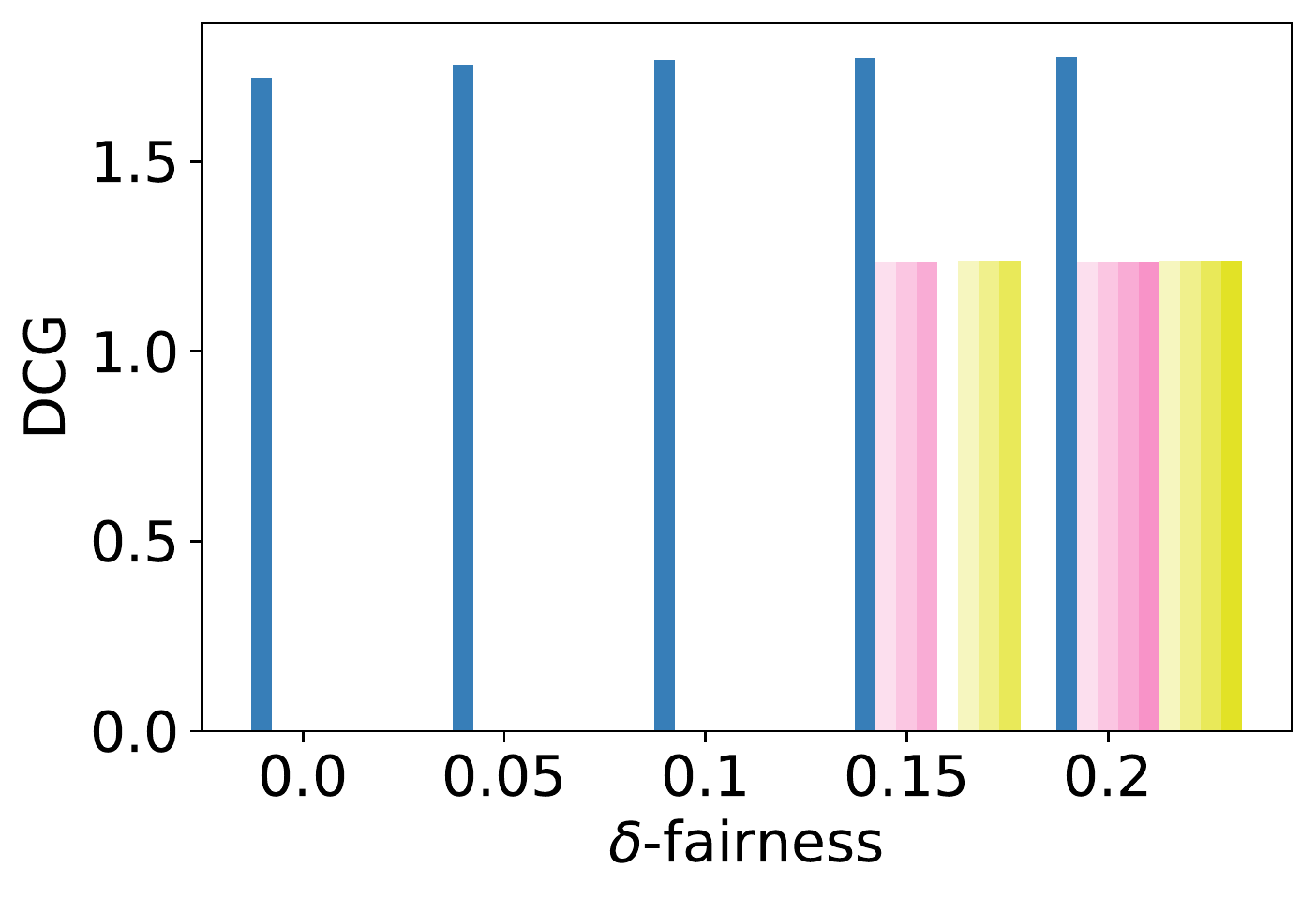}
\includegraphics[width=0.24\linewidth]{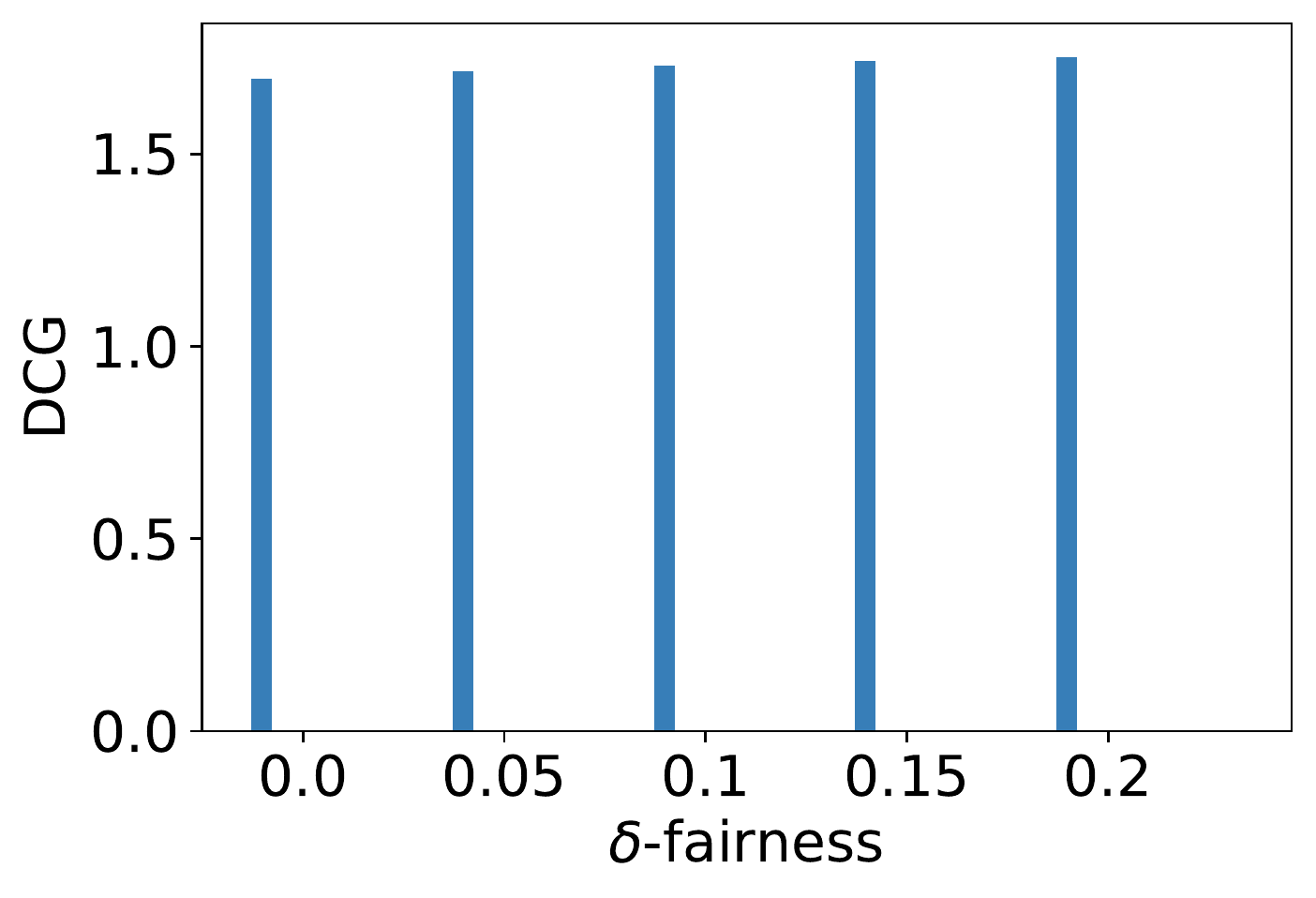}

\includegraphics[width=0.24\linewidth]{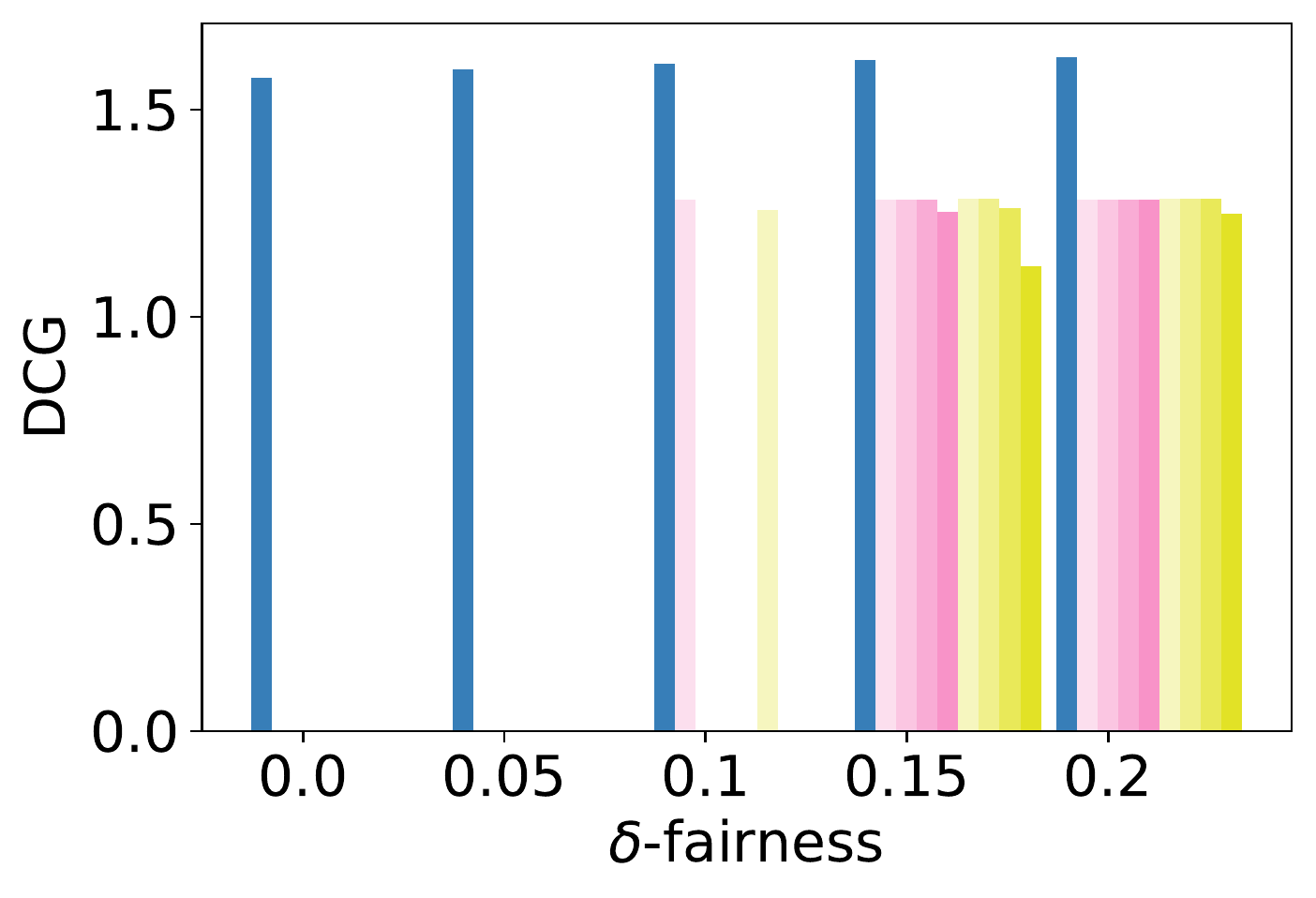}
\includegraphics[width=0.24\linewidth]{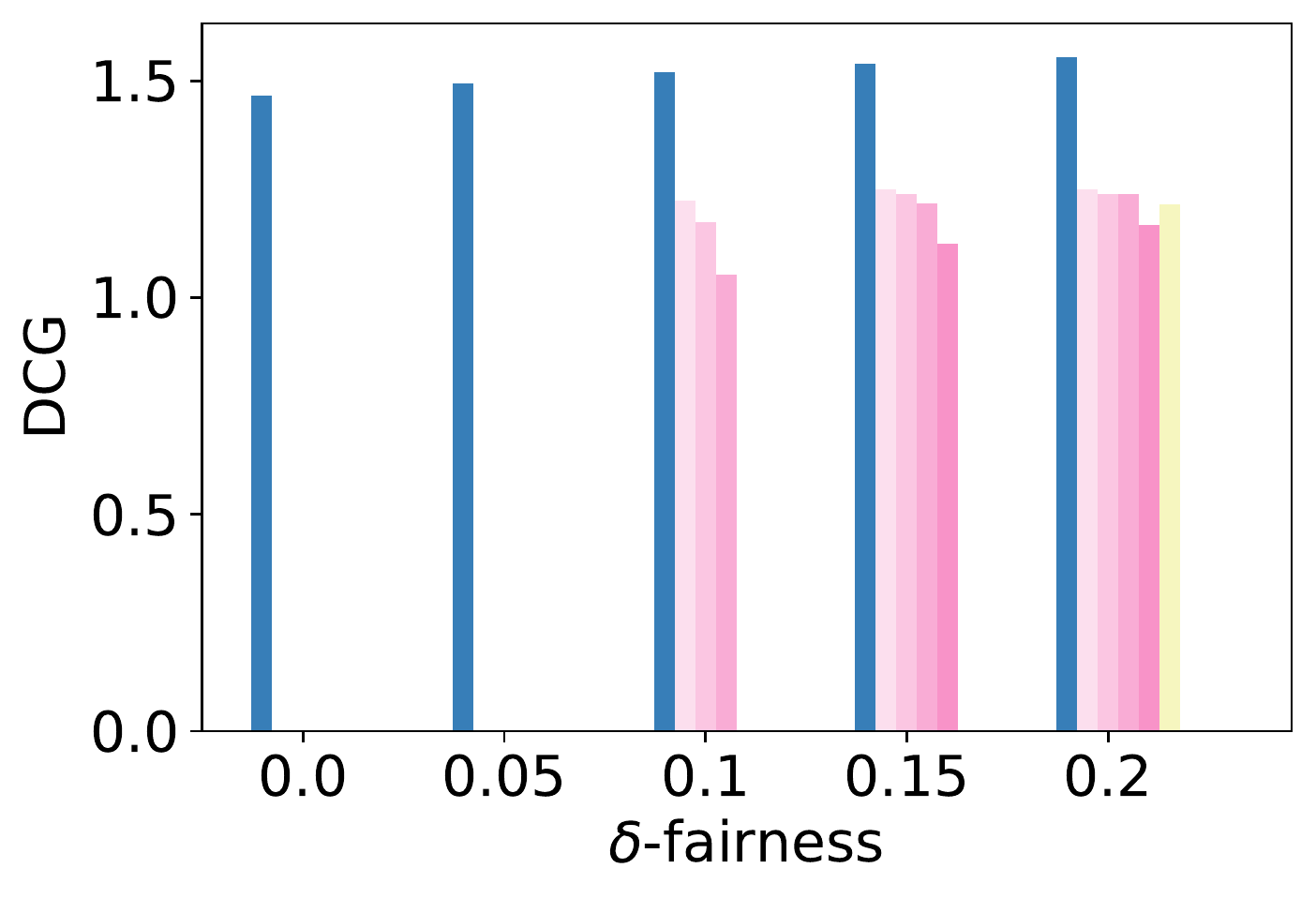}
\includegraphics[width=0.24\linewidth]{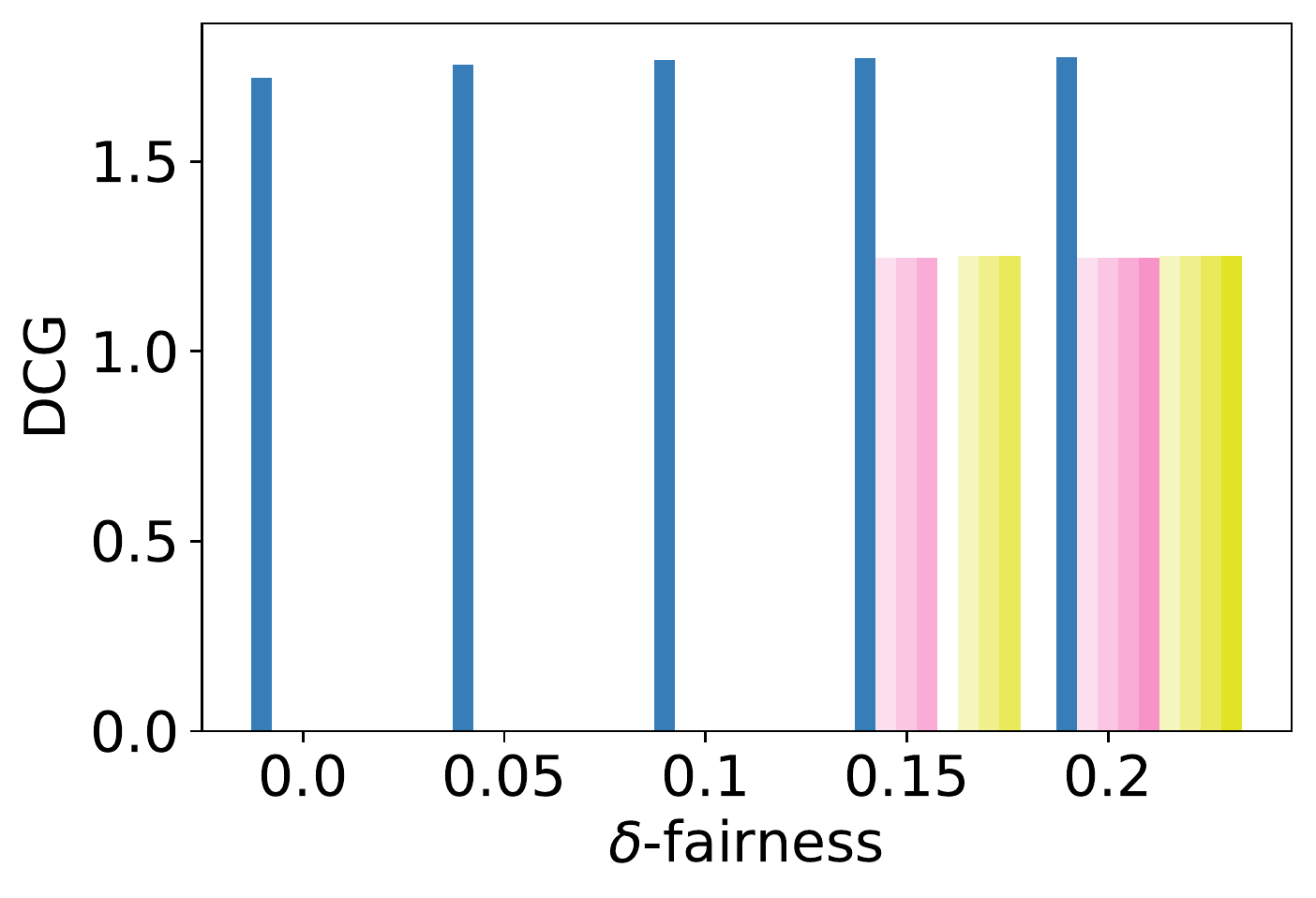}
\includegraphics[width=0.24\linewidth]{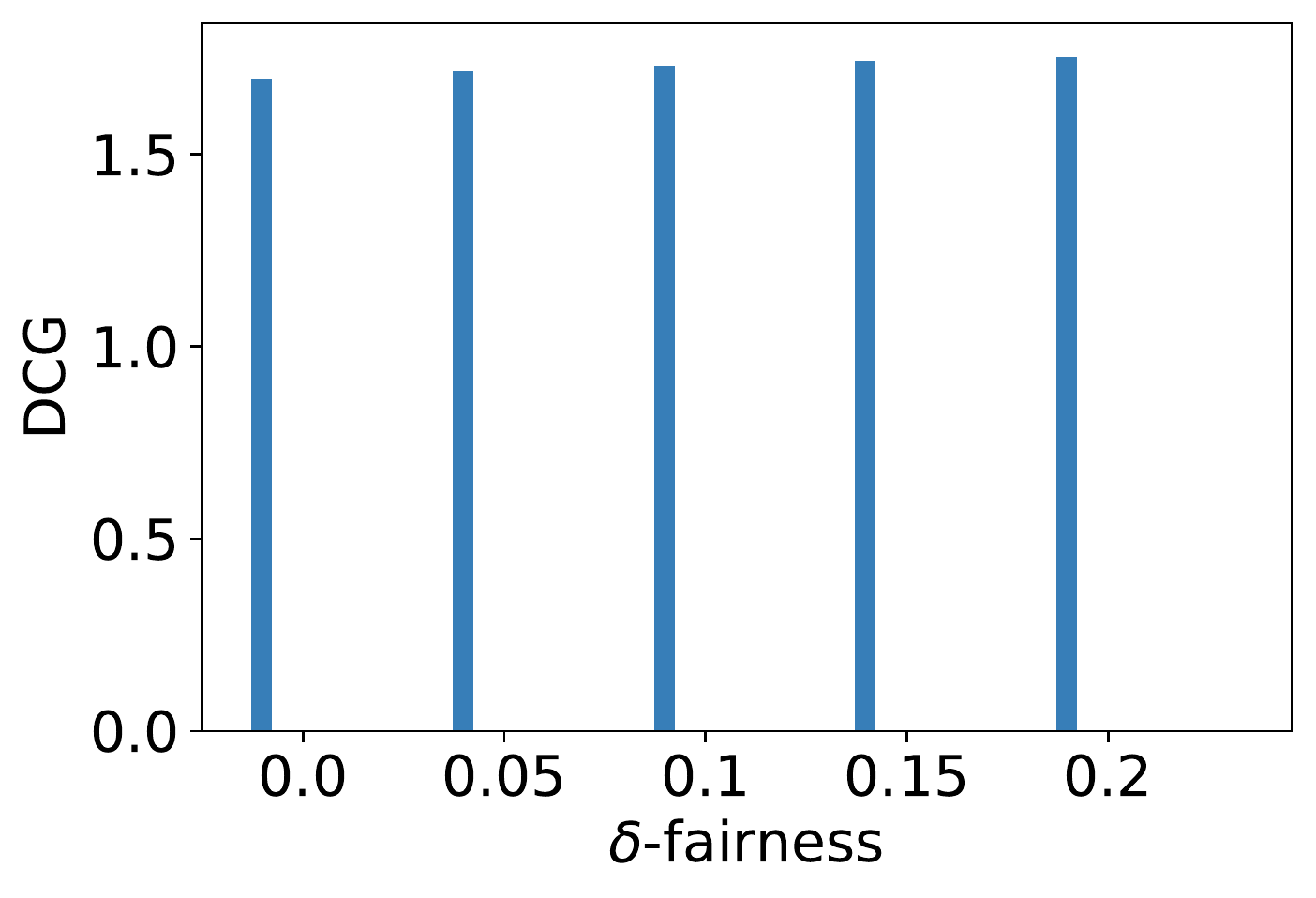}

\caption{Query level guarantees: German credit ($5K$ top row and $50K$ bottom row) unweighted ($1^{\text{st}}$ col.) and merit-weighted fairness ($2^{\text{nd}}$ col.);  
MSLR ($12K$ top row and $36K$ bottom row) unweighted ($3^{\text{rd}}$ col.) and merit-weighted fairness ($4^{\text{th}}$ col.).}
\label{fig:query_level_app}
\end{figure*}

\section{SPOFR: Implementation Details and Efficiency}
\label{app:efficiency}

For the implementation of SPOFR used to produce the experiments, the linear programming solver of Google OR-Tools was used \cite{ortools}. The algorithm has worst-case complexity that is exponential with respect to the size of the problem, but is well known to be extremely efficient in the average case \cite{bazaraa2008linear}. From an efficiency standpoint, there are three steps which must carried out to solve one instance of LP: {\bf{(1)}} Instantiation of the data structures through the solver API {\bf{(2)}} Finding a basic feasible solution to the LP problem  {\bf{(3)}} Finding the optimal solution, given a basic feasible solution.

A straightforward implementation that carries out all three steps is inefficient and can lead to long training times. Fortunately, steps {\bf{(1)}} and {\bf{(2)}} can be avoided by instantiating Model \ref{model:fair_rank} only once for each distinct (combination of) fairness constraints, and storing the respective solver states in memory. These constraints  depend only on the group identities of the items to be ranked.  Each time a query is encountered during training, the solving of Model \ref{model:fair_rank} can be resumed beginning with the solution found to the last instance sharing the same group composition. Using this solution as a hot start, only step {\bf{(3)}} is required to find the optimal policy and complete the forward pass.

For a list of length $n$, the number of distinct possible fairness constraints in the case of $2$ groups is $2^n$, which leads to unreasonable memory requirements for storing each required solver state. However, simple manipulations can be used to carry out the required calculations with only $n$ solvers held in memory, by exploiting symmetry. The group identities within an item list take the form of a binary vector $G$ of length $n$. Once sorted, there are only $n$ such distinct binary vectors.

For an input sample $(x_q)$ and corresponding group identity vector $G$, let $I_{sort} = argsort(G)$. The fairness constraints in Model \ref{model:fair_rank} are then formulated based on the sorted $G' = G[I_{sort}]$.

Let $C_{ij}$ be the objective coefficients corresponding to $\Pi_{ij}$ in Model \ref{model:fair_rank}. The rows of $C$ (corresponding to individual items) are permuted by the sorting indices of $G$:
$$ C' = C[I_{sort}][:].$$ 
Model \ref{model:fair_rank} is then solved using $G'$ and $C'$ in place of $G$ and $C$. The resulting optimal policy $\Pi'$ then need only be reverse-permuted in its rows to restore the original orders with respect to items: 
$$ \Pi = \Pi'[argsort(I_{sort})][:].$$

\begin{table}[tb]
\centering
\resizebox{\linewidth}{!}
{
    \begin{tabular}{rl lll}
    \toprule
         \textbf{Dataset size} & & \multicolumn{3}{c}{\textbf{Models}} \\
  \cmidrule(r){3-5} 
    &  & SPOFR & SPOFR* & FULTR\\
  \midrule
  \multirow{2}{*}{\textbf{5k}} & Time per Epoch (s)   & 425  & 39 & 5\\ 
                      & Epochs to Convergence  & 5 & 5 & 29 \\ 
  \multirow{2}{*}{\textbf{50k}} & Time per Epoch (s) & 2201 & 202 & 28\\ 
                      & Epochs to Convergence  & 1 & 1 & 18 \\ 
  \multirow{2}{*}{\textbf{100k}} & Time per Epoch (s) & 4338 & 398 & 60\\ 
                      & Epochs to Convergence  & 1 & 1 & 18 \\ 
  \bottomrule
\end{tabular}
}
  \caption{Runtime comparison}
\label{tab:runtime}
\end{table}

While the number of variables in Model \ref{model:fair_rank} increases quadratically with the size of the list to be ranked, its linear form ensures that it can scale efficiently to handle lists of reasonable size. Linear programming problems are routinely solved with millions of variables \cite{bazaraa2008linear}, and additionally are well-suited to benefit from hot-starting, when solutions to similar problem instances are known. Fortunately, the Smart Predict-and-Optimize framework is particularly amenable to hot-starting. Since a LP instance for each data sample must be solved at each epoch, a feasible solution to each LP is available from the previous epoch, corresponding to the same constraints and a cost vector which changes based on updates to the DNN model parameters during training \cite{mandi2019smart}. Storing a hot-start solution to each LP instance in a training set requires memory no larger than that of training set, and as the model weights converge, these hot-starts are expected to be very close to the optimal policies for each LP. The implementation described in this paper does not optimize the use of hot-starts since the use of the cached models, as described in this section, already provided large speedups, rendering training times required for SPOFR to replicate the benchmark evaluations of previous works very low. 

SPOFR and its improved implementation (SPOFR*) as described above are compared to FULTR with respect to runtime on German Credit datasets. Table \ref{tab:runtime} records the number of epochs required to converge on average over the hyperparameter search, along with average computation time per epoch. While DELTR reaches convergence relatively efficiently, it cannot produce competitive fair ranking results as shown in Section \ref{sec:experiments}; thus only SPOFR and FULTR are compared. Runtimes are reported as observed Intel(R) Xeon(R) Platinum 8260 CPU @ 2.40GHz.

\section{Experimental Setting}
\label{app:results}

\subsection{Partial-Information Setting} 
\label{sec:partial_info}
In the Full-Information
 setting, all relevance scores are typically elicited via expert
 judgments, which may be infeasible to obtain. 
 A more practical and common setting is to utilize the implicit feedback 
 from users (e.g., clicks, views) collected in an existing LTR system 
 as alternative relevance signals. However, this implicit feedback is 
 biased and cannot be directly aligned with true relevance. 
 Note that for a given query $q$, a user click is observed exclusively 
 if the user examined an item $i$ and the item was found relevant by the 
 underlying LTR model.
 Let $c_i$ and $o_i$ be a binary variables denoting, respectively, 
 whether the item $i$ was clicked and whether it was examined 
 by the user; Then, $c_i = o_i \cdot y_i$ holds. This is called
 the \textit{Partial-Information} setting~\cite{PropSVMRank:JoachimsSS17}. 

Of course, it is not appropriate to directly replace labels $y_i$ with 
$c_i$ in Equation~\eqref{eq:policy_utility}. To learn an unbiased ranking
policy in the Partial-Information setting, the implicit feedback data
must be \emph{debiased}. A widely adopted method is to use \emph
{Inverse Propensity Scoring (IPS)}, which introduces the
``propensity'' $p(o_i=1)$ of when the implicit feedback was logged.
The propensity can be modelled in various ways \cite
{PropSVMRank:JoachimsSS17, FangAJ19:ContextPBM}. The most common one
is the position-based examination model, where the propensity depends
on only the position of item $i$ in the ranking when the click was
logged. This means $p(o_{i}=1)=v_{\sigma(i)}$, where $v_k$ denotes
the examination probability at position $k$. These position biases
$v_k$ can be estimated with swap experiments~\cite
{PropSVMRank:JoachimsSS17} or intervention harvesting~\cite
{EsitimationNoIntrusive:AgarwalZWLNJ19}. With knowledge of the
propensity,  the  unbiased estimator for $\Delta$ becomes ~\cite
{PropSVMRank:JoachimsSS17}:
\begin{equation}
  \widehat{\Delta}(\sigma, c)=\sum_{i:c_i=1}\frac{ \Delta \left(\sigma
   , y_q  \right)}{p(o_{i}=1|\sigma)}.
\end{equation} The estimator is unbiased if all propensities are
 bounded away from zero~\cite{PropSVMRank:JoachimsSS17}.  
 Replacing $\Delta(\sigma, y_q)$ in Equation~\eqref
 {eq:policy_utility} with $\widehat{\Delta}(\sigma, c)$, leads to the unbiased utility estimator used to learn
 from implicit feedback in the Partial-Information setting.

\subsection{Hyper-parameters}
\label{app:hyper}

The prediction of item
 scores is the same for each model, with a single neural network
 which acts at the level of individual feature vectors as described
 in Section \ref{sec:predict}. The number of layers in each case is
 the maximum possible when each subsequent layer is halved in size;
 this depends on the length of a dataset's item feature vectors,
 which constitute the model inputs. Hyperparameters were selected as
 the best-performing on average among those listed in Table \ref{app:hyper}). 
 {Final hyperparameters for each model are
 as stated also in Table \ref{tab:hyperparams}, and Adam optimizer is
 used in the production of each result.} Asterisks (*) indicate that 
 there is no option for a final value, as all values of each parameter 
 are of interest in the analysis of fairness-utility tradeoff, as reported 
 in the experimental setting Section.
 
\begin{table}[tb]
\centering
\resizebox{\linewidth}{!}
{
\begin{tabular}{rl llll}
\toprule
  Hyperparameter    & \multicolumn{1}{c}{Min} 
  & \multicolumn{1}{c}{Max} & 
  \multicolumn{3}{c}{Final Value} \\
  \cmidrule(r){4-6} 
  & & & SPOFR & FULTR & DELTR\\
  \midrule
  learning rate   & $1e^{-6}$ & $1e^{-3}$ & $\bm{1e^{-5}}$ & $\bm{2.5e^{-4}}$ & $\bm{2.5e^{-4}}$ \\ 
  violation penalty $\lambda$   & 0 & 100 & \textbf{N/A} & \textbf{*} & \textbf{*}\\ 
  allowed violation $\delta$    & 0 & 0.4 & \textbf{*} & \textbf{N/A} & \textbf{N/A}\\ 
  entropy regularization decay  & 0.1 & 0.5 & \textbf{N/A} & $\bm{0.3}$ & \textbf{N/A}\\ 
  batch size  & 4 & 64 & \textbf{64} & \textbf{16} & \textbf{16}\\
  \bottomrule
\end{tabular}
}
  \caption{Hyperparameters}
\label{tab:hyperparams}
\end{table}

\subsection{Additional Experiments}
\label{app:additional_experiments} 
  Additional results are provided
 which investigate the effect of dataset size on the performance of
 SPOFR. For German Credit data, sizes $5k$, $50k$ and $100k$ are used
 along with $12k$, $36k$ and $120k$ for MSLR data. In both cases, the
 size of a dataset represents the number of training 'clicks'
 gathered by the click simulation. Figure \ref
 {fig:spo_dcg_fairness_sizes} indicates that the size of the training
 set does not affect test accuracy at convergence. In the case of
 merit-weighted fairness, the slight divergence in utility can be
 attributed to the difference in relative merit calculated on each
 differently-sized dataset. Figure \ref{fig:query_level_app} shows
 that both baseline methods benefit in terms of query-level fairness
 from increased dataset size, but the effect on SPOFR is negligible.
 Note also that FULTR reports an increase in accuracy as training
 samples are added \cite{yadav2021policy}. This may indicate an
 advantage in terms of data-efficiency attributable to SPOFR. 

\end{document}